\documentclass{article}

\usepackage{amsmath, amssymb, amsthm}
\usepackage{mathtools}
\usepackage{microtype}
\usepackage{graphicx}
\usepackage{subfigure}
\usepackage{booktabs} 
\usepackage{array}
\usepackage{xcolor}
\usepackage{algorithmic}

\usepackage{hyperref}

\newcommand{\deltat}{{\delta\hspace{-.1em}t}}

\newcommand{\dbt}{d\hspace{-.1em}B_t}
\newcommand{\reward}{{r}}

\newcommand{\E}{\mathbb{E}}
\newcommand{\BP}{\mathbb{P}}
\newcommand{\bigO}{O}
\newcommand{\smallo}{o}

\newcommand{\mathdef}{\colon\hspace{-.6em}=}

\newcommand{\deq}{\mathrel{\mathop{:}}=}

\newcommand{\bb}[1]{\mathbb{#1}}
\newcommand{\from}{\colon}
\newcommand{\eps}{\varepsilon}

\DeclareMathOperator*{\argmax}{argmax}

\usepackage[accepted]{icml2019}

\icmltitlerunning{Making Deep Q-learning methods robust to time discretization}

\theoremstyle{plain}
\newtheorem{theorem}{Theorem}
\newtheorem{lemma}{Lemma}

\begin{document}

\twocolumn[
\icmltitle{Making Deep Q-learning Methods Robust to Time Discretization}



\icmlsetsymbol{equal}{*}

\begin{icmlauthorlist}
	\icmlauthor{Corentin Tallec}{tau}
	\icmlauthor{L\'eonard Blier}{tau,fair}
	\icmlauthor{Yann Ollivier}{fair}
\end{icmlauthorlist}

\icmlaffiliation{tau}{TAckling the Underspecified, Universit\'e Paris Sud}
\icmlaffiliation{fair}{Facebook Artificial Intelligence Research}
%
\icmlcorrespondingauthor{Corentin Tallec}{corentin.tallec@u-psud.fr}
\icmlcorrespondingauthor{L\'eonard Blier}{leonardb@fb.com}

\icmlkeywords{Reinforcement learning}

\vskip 0.3in
]



\printAffiliationsAndNotice{}  

\begin{abstract}
Despite remarkable successes, \emph{Deep Reinforcement Learning}
(DRL) is not robust to hyperparameterization, implementation
details, or small environment changes (\citealt{drl_matter},
\citealt{drl_matter_bis}).  Overcoming such sensitivity is key to
making DRL applicable to real world problems.  In this paper,
we identify
sensitivity to \emph{time discretization} in near
continuous-time environments as a critical factor; this covers, e.g., changing the number of
frames per second, or the action frequency of the controller.
Empirically, we find that
$Q$-learning-based approaches such as
\emph{Deep $Q$-learning}~\cite{dqn} and \emph{Deep Deterministic
Policy Gradient}~\cite{ddpg} 
collapse with small time steps. Formally, we prove that
$Q$-learning does not exist in
continuous time.
We detail a principled way to build an
off-policy RL algorithm that yields similar performances over a wide range
of time
discretizations, and confirm this robustness empirically. 
\end{abstract}


\section{Introduction}
\label{sec:intro}
In recent years, \emph{Deep Reinforcement Learning} (DRL) approaches have
provided impressive results in a variety of domains, achieving superhuman
performance with no expert knowledge in perfect information zero-sum
games~\cite{alphazero}, reaching top player level in video
games~(\citealt{openai_five}, \citealt{dqn}), or learning dexterous manipulation
from scratch without demonstrations~\cite{hand_control}. 
In spite of those successes, DRL approaches are
sensitive to a number of factors, including hyperparameterization,
implementation details or small changes in the environment
parameters~(\citealt{drl_matter}, \citealt{drl_matter_bis}). This sensitivity,
along with sample inefficiency, largely prevents DRL from being applied in real
world settings. Notably, high sensitivity to environment parameters prevents
transfer from imperfect simulators to real world scenarios.

In this paper we focus on the sensitivity to time discretization of DRL
approaches, such as what happens when an agent receives $50$ observations
and is expected to take $50$ actions per second instead of $10$.
In principle, decreasing time discretization, or equivalently
shortening reaction time, should only improve agent performance.
Robustness to time discretization is especially relevant in
\emph{near-continuous} environments, which includes most continuous
control environments, robotics, and many video games.  

Standard approaches based on estimation of state-action
value functions, such as
\emph{Deep $Q$-learning} (DQN,~\citealt{dqn}) and \emph{Deep deterministic policy
gradient} (DDPG,~\citealt{ddpg}) are not at all robust to changes in time discretization. This is shown experimentally in Sec.~\ref{sec:exp}. 
Intuitively, as the discretization timestep
decreases, the effect of individual actions on the total return decreases
too: $Q^*(s,a)$ is the value of playing action $a$ then playing optimally,
and if $a$ is only maintained for a very short time its advantage over
other actions will be accordingly small. (This occurs even with a
suitably adjusted decay factor $\gamma$.)
If the discretization timestep becomes infinitesimal, the effect of every
individual action vanishes: there is no continuous-time $Q$-function
(Thm.~\ref{th:q-cont}), hence the poor performance of $Q$-learning with
small time steps. These statements can be fully formalized in the framework
of continuous-time reinforcement learning (Sec.~\ref{sec:framework})
\cite{cont_rl, adv_upd}. 

We focus on continuous time because this leads to a clear theoretical framework,
but our observations make sense in any setting in which the value
results from taking a large number of small individual actions. Our
results suggest standard $Q$-learning will fail in such settings without a
delicate balance of hyperparameter scalings and reparameterizations.

We are looking for an algorithm that would be as invariant as possible to
changing the discretization timestep. Such an 
algorithm should remain viable when this timestep is small, and in particular admit a
continuous-time
limit when the discretization timestep goes to $0$.  This
leads to precise design choices in term of agent architecture, exploration
policy, and learning rates scalings.  The resulting algorithm is shown to
provide better invariance to time discretization than
vanilla DQN or DDPG, on many environments (Sec.~\ref{sec:exp}).  On a new environment, as soon as the
order of magnitude of the time discretization is known, our analysis readily
provides relevant scalings for a number of hyperparameters.

Our contribution is threefold:
\begin{itemize} 
\item Building on \cite{adv_upd}, we formally show that the $Q$-function collapses to the $V$-function in near-continuous time, and thus that
    standard $Q$-learning is ill-behaved in this setting.
  \item Our analysis of properties in the continuous-time limit leads to a robust off-policy algorithm. In particular,
    it provides insights on architecture design, and constrains
    exploration schemes and learning rates scalings.
  \item We empirically show that standard $Q$-learning methods are not
  robust to changes in time discretization, exhibiting degraded
  performance, while our algorithm demonstrates
  substantial robustness.
\end{itemize}


\section{Related Work}
\label{sec:related}

Our approach builds on \cite{adv_upd},  who identified the collapse of
$Q$-learning for small time steps and, as a solution, suggested the Advantage
Updating algorithm, with proper scalings for the $V$ and
advantage parts depending on timescale $\deltat$; testing was only done on a quadratic-linear problem.

We expand on \cite{adv_upd} in several directions. First, we modify the
algorithm by
using a different normalization step for $A$, which forgoes the need to
learn the normalization itself, thanks to the parameterization
\eqref{eq:Aparam}. Second, we test Advantage Updating for the first time on a
variety on RL environments using deep networks, establishing Deep
Advantage Updating as a viable algorithm in this setting.  Third,
we provide formal proofs in a general setting for the collapse of
$Q$-learning when the timescale $\deltat$ tends to $0$,
and for the non-collapse of Advantage Updating with the proper scalings.
Fourth, we also discuss how to obtain $\deltat$-invariant exploration. Fifth, we
provide stringent experimental tests of the actual robustness to changing
$\deltat$.

Our study focuses on off-policy algorithms. Some on-policy algorithms,
such as A3C~\cite{a3c}, PPO~\cite{ppo} or TRPO~\cite{trpo} may be time discretization
invariant with specific setups. This is out of the scope of our work and
would require a separate study.

\cite{dueling_nets} also use a parameterization separating
the value and advantage components of the $Q$-function. But contrary to
\cite{adv_upd}'s Advantage Updating, learning is still done in a standard
way on the $Q$-function obtained from adding these two components. Thus this
approach reparameterizes $Q$ but does not change scalings
and does not result in an invariant algorithm for small $\deltat$.


The problem studied here is a continuity effect quite distinct from
multiscale RL approaches: indeed the issue arises even if there
is only one timescale in the environment. Arguably, a small
$\deltat$ can be seen as a mismatch between the algorithm's
timescale and the physical system's timescale, but the collapse of the
$Q$ function to the $V$ function is an intrinsic mathematical phenomenon arising from time
continuity.

Reinforcement learning has been studied from a mathematical perspective
when time and space are both continuous, in connection with optimal
control and the Hamilton--Jacobi--Bellman (HJB) equation (a PDE which
characterizes the value function for continuous space-time). Explicit
algorithms for continuous
space-time can be found in
\cite{cont_rl,MunosBourgines98} (see also the references therein).
\cite{MunosBourgines98} use a grid approach to provably solve the HJB
equation when discretization tends to $0$ (assuming every state in the
grid is visited a large number of times). However,
the resulting algorithms are
impractical \cite{cont_rl} for larger-dimensional problems.
\cite{cont_rl} focusses on algorithms specific to the continuous
space-time case, including advantage updating and modelling
the time derivative of the environment.

Here on the other hand we focus on generic deep RL algorithms that can handle
both discrete and continuous time and space, without collapsing in
continuous time, thus being robust to arbitrary timesteps.



\section{Near Continuous-Time Reinforcement Learning}
\label{sec:framework}

Many reinforcement learning environments are not intrinsically
time-discrete, but discretizations of an underlying continuous-time
environment. For instance, many simulated control environments, such as
the Mujoco environments~\cite{ddpg} or OpenAI Gym classic control
environments~\cite{gym}, are discretizations of continuous-time control
problems.  In simulated environments, the time discretization is fixed by
the simulator, and is often used to approximate an underlying
differential equation.  In this case, the timestep may correspond to the
number of frames generated by second.  In real world environments,
sensors and actuators have a fixed time precision: cameras can only
capture a fixed amount of frames per second, and physical limitations
prevent actuators from responding instantaneously. The quality of these
components thus imposes a lower bound on the discretization timestep. As
the timestep $\deltat$ is 
largely a constraint imposed by computational ressources, we
would expect that decreasing $\deltat$ would only improve the performance
of RL agents (though it might make optimization
harder).  RL algorithms should, at least, be resilient to a change of
$\deltat$, and should remain viable when $\deltat \rightarrow 0$.
Besides, designing a time discretization invariant algorithm could
alleviate tedious hyperparameterization by providing better defaults for
time-horizon-related parameters.

We are thus interested in the behavior of RL
algorithms in discretized environments, when the discretization timestep
is small. We will refer to such environments as \emph{near-continuous
environments}.
A formalized view of near-continuous environments is
given below, along with $\deltat$-dependent definitions of return, discount
factor and value functions, that converge to well defined
continous-time limits. The state-action value
function is shown to collapse to the value function as $\deltat$ goes to $0$.
Consequently there is no $Q$-learning in continuous time, foreshadowing
problematic behavior of $Q$-learning with small timesteps.

\subsection{Framework}




Let ${\cal S} = \mathbb{R}^d$ be a set of states, and ${\cal A}$ be a set of
actions. Consider the continuous-time \emph{Markov Decision Process} (MDP) defined by the differential equation 
        \begin{equation}
	\label{eq:diffusion}
	ds_t/dt = F(s_t, a_t) 
              \end{equation}
where $F\colon {\cal S}\times{\cal A}\rightarrow {\cal S}$ describes
the dynamics of the environment. The agent interacts with the environment through a deterministic policy function
$\pi \colon {\cal S} \rightarrow {\cal A}$, so that $a_t=\pi(s_t)$.
Actions can be discrete or continuous. 
For simplicity we assume here that both the dynamics and exploitation policy are
deterministic; \footnote{
      We believe the results presented here hold more
      generally, assuming states
      follow a \emph{stochastic} differential equation \begin{equation}
	      ds = F(s, a)dt  + \Sigma(s, a)\dbt
	      \label{eq:sde}
      \end{equation} with
      $B_t$ a multidimensional Brownian motion and $\Sigma$ a covariance matrix. A
      formal treatment of SDEs is beyond the scope of this paper.}
 the exploration policy will be random, but care must be
taken to define proper random policies 
in continuous time, especially with discrete actions (Sec.~\ref{subsec:explo}).

For any timestep $\deltat>0$, we can define an MDP ${\cal
M}_\deltat = \langle {\cal S}, {\cal A}, T_{\deltat}, r_\deltat,
      \gamma_\deltat\rangle$ as a discretization of the continuous-time MDP with
      \emph{time discretization} $\deltat$. The 
      transition function of a state $s$
      is 
      the
      state obtained
      when starting at $s_0 = s$ and maintaining $a_t=a$ constant for a time
      $\deltat$.
      This corresponds to an agent evolving in the continuous
      environment \eqref{eq:diffusion}, but 
      only making observations and choosing actions every $\deltat$. The
      rewards and decay factor are specified below. We
      call such an  MDP ${\cal M}_\deltat$
      \emph{near-continuous}.      



A necessary condition for robustness of an algorithm for
near-continuous time MDPs is 
to remain viable when $\deltat \rightarrow 0$. Thus we will try to
make sure the various quantities involved converge to meaningful
limits when $\deltat\to 0$. 

We give semi-formal statements below; the full statements, proofs, and
technical assumptions (typically, differentiability assumptions)
can be found in the supplementary material.

%


\paragraph{Return and discount factor.}
Suitable 
$\deltat$-dependent scalings of the discount factor $\gamma_\deltat$
and reward $r_\deltat$ are as follows. These definitions fit the discrete case
when $\deltat=1$, and provide well-defined, non-trivial returns
and value functions when $\deltat$ goes to $0$.

For a continuous MDP and a continuous trajectory $\tau = (s_t, a_t)_t$,
the return is defined as~\cite{cont_rl}
\begin{equation}
\label{eq:continuous-return}
R(\tau) \deq \int_{t=0}^\infty\gamma^t\,r(s_t, a_t)\,dt.
\end{equation}

A natural time discretization is obtained  by defining the discretized
return $R_\deltat$ of the MDP $\mathcal{M}_{\deltat}$ as
\begin{equation}
\label{eq:discretized-return}
R_\deltat(\tau) \deq \sum_{k=0}^\infty\gamma^{k\deltat}\,
r(s_{k\deltat}, a_{k\deltat})\,\deltat
\end{equation}
and the discretized return will correspond to the continuous-time return
if we set the decay factor $\gamma_\deltat$ and rewards $r_\deltat$ of
the discretized MDP ${\cal M}_\deltat$ to
\begin{align}
\label{eq:def-gamma}
\gamma_\deltat \deq \gamma^\deltat,\qquad
r_\deltat \deq \deltat . r.
\end{align}

      \paragraph{Physical time vs algorithmic time, time horizon.}
        In near-continuous environments, there are two notions of time:
the algorithmic time $k$ (number of steps or actions taken), and the
physical time $t$ (time spent
in the underlying continuous time environment), related via
$t=k.\deltat$.

The time horizon is, informally, the time range over
which the agent optimizes its return.
As a rule of thumb,
the time horizon of an agent with discount factor $\gamma$ is
of order $\frac{1}{1 - \gamma}$ steps; beyond that, the decay factor
kicks in and the influence of further rewards becomes small.

The definition~\eqref{eq:def-gamma} of the decay factor $\gamma_\deltat$ in
near-continuous environments keeps the time horizon constant in
\emph{physical} time, by making $\gamma_\deltat$ close to $1$ in
algorithmic time. Indeed, physical time horizon is $\deltat$ times the
algorithmic time horizon, namely
\begin{equation}
          \label{eq:time-horizon}
          \frac{\deltat}{1-\gamma^\deltat}
          = - \frac{1}{\log \gamma} + \bigO(\deltat)
          \approx \frac{1}{1-\gamma},
        \end{equation}
which is thus stable when $\deltat\to 0$. 
On the other hand,
if $\gamma_\deltat$ was left constant
as $\deltat$ goes to $0$, the corresponding time horizon in physical time
would be $\approx \frac{\deltat}{1 - \gamma}$ which goes to $0$ when
$\deltat$ goes to $0$: such an agent would be increasingly short-sighted
as $\deltat\to 0$.

In the following, we use the suitably-scaled decay factor
\eqref{eq:def-gamma} both for Deep Advantage Updating and for the
classical deep $Q$-learning baselines.
%

\paragraph{Value function.}
The return \eqref{eq:continuous-return} leads to the
following continuous-time value function
\begin{align}
  V^\pi(s) &= \E_{\tau\sim\pi}\left[R(\tau) \mid s_0 = s\right]\\
  &= \E_{\tau\sim\pi}\left[\int\limits_{0}^\infty \gamma^{t}\, r(s_t,
  a_t) \,dt \mid s_0 = s\right].
\end{align}
Meanwhile, the value function (in the ordinary sense) of the discrete MDP ${\cal
M}_\deltat$ is
\begin{align}
  V^\pi_\deltat(s) &= \E_{\tau\sim\pi}\left[R_\deltat(\tau) \mid s_0 = s\right]\\
  &= \E_{\tau\sim\pi}\left[\sum\limits_{k=0}^\infty
  \gamma^{k\deltat}\,r(s_{k\deltat}, a_{k\deltat})\,\deltat \mid s_0 = s\right]
\end{align}
which obeys the Bellman equation
\footnote{
If the continuous MDP follows the dynamics \eqref{eq:diffusion}, the limit of
the Bellman equation \eqref{eq:bellman} for $V^\pi_\deltat$ when $\deltat \rightarrow 0$ is
the \emph{Hamilton--Jacobi--Bellman} equation on $V^\pi$ \cite{cont_rl},
namely,
  $r + \nabla_s V^\pi \cdot F = - \log(\gamma) V^\pi$.
}
\begin{equation}
  \label{eq:bellman}
  V^\pi_\deltat(s) = \reward(s, \pi(s))\,\deltat + \gamma^{\deltat}\,
  \E_{s_{(k+1)\deltat} | s_{k\deltat} = s} V^\pi_\deltat(s_{(k+1)\deltat})
\end{equation}

When the timestep tends to $0$, one converges to the other.
  \begin{theorem}
    \label{th:conv-value}
Under suitable smoothness assumptions, $V^\pi_\deltat(s)$ converges to
$V^\pi(s)$ when $\deltat\to 0$.
   \end{theorem}

\subsection{There is No $Q$-Function in Continuous Time}

Contrary to the value function, the action-value function $Q$ is ill-defined
for continuous-time MDPs. More precisely, the $Q$-function  collapses to the
$V$-function when $\deltat \rightarrow 0$. In near continuous time, the effect of
individual actions on the $Q$-function is of order $\bigO(\deltat)$. This
will make 
ranking of actions difficult, especially with an approximate
$Q$-function.
This argument appears informally in \cite{adv_upd}.
Formally:
\begin{theorem}
  \label{th:q-cont}
  Under suitable smoothness assumptions,
The action-value function of a near-continuous MDP is related to its
value function via
\begin{equation}
	\label{eq:q-discr}
	Q^\pi_\deltat(s, a) = V^\pi_\deltat(s) + \bigO\left(\deltat\right)
\end{equation}
when $\deltat \to 0$, for every $(s,a)$.
\end{theorem}

As a consequence, in exact continuous time,
$Q^\pi$ is equal to $V^\pi$:  it does not bear \emph{any} information on the ranking of actions, and
thus cannot be used to select actions with higher returns.
There is no continuous-time $Q$-learning.

\begin{proof}
The discrete-time $Q$-function of the MDP $\mathcal{M}_\deltat$ satisfies the Bellman equation
  \begin{align}
    Q^\pi_\deltat(s, a) &= r(s, a)\,\deltat + \gamma^\deltat \E_{s'|s,
    a}\left[V^\pi_\deltat(s')\right] .
  \end{align}
The dynamics \eqref{eq:diffusion} of the environment yields
  \begin{equation}
  s' = s + F(s, a) \,\deltat + \smallo(\deltat).
  \end{equation}
  Assuming that $V^\pi_\deltat$ is continuously differentiable with
  respect to the state, and that its derivatives are uniformly bounded, we
  obtain,
  \begin{align}
  V^\pi_\deltat(s') &= V^\pi_\deltat(s) + \nabla_s V^\pi_\deltat(s)\cdot
  F(s, a)\,\deltat + \smallo(\deltat)\\
		    &= V^\pi_\deltat(s) + \bigO(\deltat)
  \end{align}
  Expanding $V^\pi_\deltat(s')$ into $Q^\pi_\deltat$ yields
  \begin{align}
  Q^\pi_\deltat(s, a) &= r(s, a)\,\deltat + \gamma^\deltat (V^\pi_\deltat(s) + \bigO(\deltat))\\
		      &= \bigO(\deltat)+(1 + \bigO(\deltat)) (V^\pi_\deltat(s) + \bigO(\deltat)) \\
		      &= V^\pi_\deltat(s) + \bigO(\deltat).
  \end{align}
which ends the proof.
  \end{proof}


\section{Reinforcement Learning with a Continuous-Time Limit}

We now define a discrete algorithm with a well-defined continuous-time
limit.  It relies on three elements: defining and learning a quantity
that still contains information on action rankings in the limit, using
exploration methods with a meaningful limit, and scaling learning rates
to induce well-behaved parameter trajectories when $\deltat$ goes to $0$.


\subsection{Advantage Updating}
\label{subsec:reparam}

As seen above, there is no continuous time limit to $Q$-learning, because
$Q^\pi$ becomes independent of actions and thus cannot be
used to select actions.  With small but nonzero $\deltat$,
$Q^\pi_\deltat$ still depends on actions, and could still be used to
choose actions. However, when
approximating $Q^\pi_\deltat$, if the approximation error is much larger
than $\bigO(\deltat)$, this error dominates, the ranking of
actions given by the approximated $Q^\pi_\deltat$ is likely to be erroneous.

To define an object which contains the same information on actions as
$Q^\pi_\deltat$, but admits a learnable action-dependent limit, it is
natural to define \cite{adv_upd}
\begin{align}
	A^\pi_\deltat(s, a) &\mathdef \frac{Q^\pi_\deltat(s,a) - V^\pi_\deltat(s)}{\deltat},
    \label{eq:adv}
\end{align}
a rescaled version of the advantage function, as the difference between between
$Q^\pi_\deltat(s, a)$ and $V^\pi_\deltat(s)$ is of order
$\bigO(\deltat)$. This amounts to splitting $Q$ into value and advantage,
and observing that these scale very differently when $\deltat\to 0$.

Contrary to the $Q$-function,
this
rescaled advantage function converges when $\deltat\to 0$
to a non-degenerate action-dependent quantity.

\begin{theorem}
\label{thm:Alimit}
Under suitable smoothness assumptions, $A^\pi_\deltat(s,a)$ has a limit
$A^\pi(s,a)$ when $\deltat\to 0$. The limit keeps information about
actions: namely, if a policy $\pi'$ strictly dominates $\pi$, 
then
$A^\pi(s,\pi'(s))>0$ for some state $s$.
\end{theorem}



\paragraph{Learning $A^\pi$.} The discretized $Q$-function rewrites as
\begin{equation}
	Q^\pi_\deltat(s, a) = V^\pi_\deltat(s) + \deltat A^\pi_\deltat(s, a).
	\label{eq:reparam_q_pi}
\end{equation}
A natural way to approximate $V^\pi_\deltat$ and $A^\pi_\deltat$ is to apply
\emph{Sarsa} or $Q$-learning to a reparameterized $Q$-function approximator
\begin{equation}
	Q_\Theta(s, a) \deq V_{\theta}(s) + \deltat A_{\psi}(s, a).
\end{equation}
with $\Theta \deq (\theta, \psi)$. At initialization, if both $V_{\theta}$ and
$A_{\psi}$ are initialized independently of $\deltat$, this parameterization
provides reasonable scaling of the contribution of actions versus states
in $Q$.
Our goal is for $V_\theta$ to approximate $V^\pi_\deltat$ and for
$A_{\psi}$ to approximate $A^\pi_\deltat$.

Still, this reparameterization does not, on its own, guarantee that $A$
correctly approximates $A^\pi_\deltat$ if
$Q_\Theta$ approximates $Q^\pi_\deltat$
.
Indeed, for any given pair $(V_\theta,\,A_\psi)$, the pair $(V_\theta(s) -
f(s),\,A_\psi(s,a)+f(s)/\deltat)$ (for an arbitrary $f$)
yields the exact same function $Q_\Theta$. This new $A_\psi$ still defines 
the same ranking of actions, yet this phenomenon might cause
numerical problems or instability of $A_\psi$ when $\deltat\to 0$, and prevents direct
interpretation of the learned $A_\psi$.
To enforce identifiability of $A_\psi$, one must enforce the consistency equation
\begin{equation}
	V^\pi_\deltat(s) = Q^\pi_\deltat(s,
	\pi(s))
\end{equation}
on the approximate $A_\psi$ and $V_\theta$. This translates to
\begin{equation}
	A_\psi(s, \pi(s)) = 0.
\end{equation}
With this additional constraint, if $Q_\Theta = Q^\pi_\deltat$, then $A_\psi =
A^\pi_\deltat$ and $V_\theta = V^\pi_\deltat$: indeed
\begin{align}
	A^\pi_\deltat(s, a) &= \frac{Q^\pi_\deltat(s,a) - V^\pi_\deltat(s)}{\deltat}\\
		    &= \frac{Q_\Theta(s,a) - Q_\Theta(s, \pi(s))}{\deltat}
		    = A_\psi(s, a).
\end{align}
In the spirit of~\cite{dueling_nets}, instead of directly parameterizing $A_\psi$,
we define a parametric function $\bar{A}_\psi$ (typically a neural network),
and use $\bar{A}_\psi$ to define $A_\psi$ as
\begin{equation}
\label{eq:Aparam}
	A_\psi(s, a) \deq \bar{A}_\psi(s, a) - \bar{A}_\psi(s, \pi(s))
\end{equation}
so that $A_\psi$ directly verifies the consistency condition.

This approach will lead to $\deltat$-robust algorithms for
approximating $A^\pi_\deltat$, from which a ranking of actions can be
derived.

\subsection{Timestep-Invariant Exploration}
\label{subsec:explo}

To obtain a timestep-invariant RL algorithm, a timestep-invariant
exploration scheme is required. 
For \emph{continuous} actions, \cite{ddpg} already introduced such a scheme, 
by adding an
\emph{Ornstein--Uhlenbeck} \cite{orn-uhl} (OU) random process to the
actions. Formally, this is defined as
\begin{equation}
	\pi^\text{explore}(s_{k\deltat}, z_{k\deltat}) \deq
	\pi(s_{k\deltat}) + z_{k\deltat}
\end{equation}
with $z_{k\deltat}$ the discretization of a continuous-time OU process,
\begin{equation}
	dz_t = - z_t \,\kappa\, dt + \sigma \,\dbt.
	\label{eq:orn_uhl}
\end{equation}
where $B_t$ is a brownian motion, $\kappa$ a stiffness parameter and
$\sigma$ a noise scaling parameter. The discretized trajectories of $z$ converge
to nontrivial continuous-time trajectories, exhibiting Brownian
behavior with a recall force towards $0$.

This exploration can be extended to schemes of the form
\begin{equation}
  \label{eq:explore}
	a_{k\deltat} = \pi^\text{explore}_\deltat(s_{k\deltat},
	z_{k\deltat})
\end{equation}
with $(z_{k\deltat})_{k\geq 0}$ a sequence of random variables independent from the $a$'s and $s$'s.
A sufficient condition for this policy to admit a continuous-time
limit is for the sequence
$z_{k\deltat}$ to converge in law to a
well-defined continuous stochastic process $z_t$ as $\deltat$ goes to $0$.

Thus, for \emph{discrete} actions we can obtain a consistent
exploration scheme by taking
$z_\deltat$ to be a discretization of an $(\#\mathcal{A})$-dimensional
continuous OU process, and setting
\begin{equation}
  \pi^\text{explore}(s_{k\deltat}, z_{k\deltat})\deq \argmax_{a}
  \left(A_\psi(s_{k\deltat}, a) + z_{k\deltat}[a]\right)
\end{equation}
where $z_{k\deltat}[a]$ denotes the $a$-th component of $z_{k\deltat}$. Namely,
we perturb the \emph{advantage values} by a random process before selecting an action. The
resulting scheme converges in continuous time to a nontrivial exploration
scheme.

On the other hand,
$\varepsilon$-greedy
exploration is likely \emph{not} to explore, i.e., to collapse to a deterministic
policy, when $\deltat$ goes to $0$.
Intuitively,
with very small $\deltat$, changing the action at random every $\deltat$ time
step just averages out the randomness due to the law of large numbers.
More precisely:

\begin{theorem}
Consider a near-continuous MDP in which an agent selects an
action according to an $\eps$-greedy policy that mixes a deterministic
exploitation policy $\pi$ with an action taken from a noise policy
$\pi^\text{noise}(a|s)$ with probability $\eps$ at each step. Then the
agent's trajectories converge when $\deltat\to 0$ to the solutions of the
\emph{deterministic} equation
\begin{equation}
d s_t/dt= (1-\eps) F(s_t,\pi(s_t))+\eps \E_{a\sim
\pi^\text{noise}(a|s)}F(s_t,a)
\end{equation}
\end{theorem}

\subsection{Algorithms for Deep Advantage Updating}
\label{subsec:algorithm}
\begin{algorithm}[ht]
  \caption{Deep Advantage Updating (Discrete actions)}
  \label{alg:dau}
\begin{algorithmic}
	\STATE \textbf{Inputs:}
	\STATE $\theta$ and $\psi$, parameters of
	$V_{\theta}$ and $\bar{A}_{\psi}$.
	\STATE $\pi^{\text{explore}}$ and $\nu_\deltat$ defining an exploration policy.
	\STATE \textbf{opt}$_V$, \textbf{opt}$_A$, $\alpha^V \deltat$ and $\alpha^A \deltat$ optimizers and learning rates.
	\STATE $\mathcal{D}$, buffer of transitions $(s, a, r, d, s')$, with $d$ the episode termination signal.
	\STATE $\deltat$ and $\gamma$, time discretization and discount factor.
	\STATE \textbf{nb\_epochs} number of epochs.
	\STATE \textbf{nb\_steps}, number of steps per epoch.
	\STATE
	\STATE Observe initial state $s^0$
	\STATE $t \gets 0$
	\FOR {$e=0, \textbf{nb\_epochs}$}
	\FOR {$j=1, \textbf{nb\_steps}$}
	\STATE $a^k \leftarrow \pi^{\text{explore}}(s^k, \nu^k_\deltat)$.
	\STATE Perform $a^k$ and observe $(r^{k+1}, d^{k+1}, s^{k+1})$.
	\STATE Store $(s^k, a^k, r^{k+1}, d^{k+1}, s^{k+1})$ in $\mathcal{D}$.
	\STATE $k \gets k + 1$
	\ENDFOR
	\FOR {$k=0, \text{nb\_learn}$}
	\STATE \text{Sample a batch of $N$ random transitions from $\mathcal{D}$}
	\STATE $Q^i \gets V_{\theta}(s^i) + \deltat\hspace{-.17em}\left(
	\bar{A}_{\psi}(s^i, a^i) - \max\limits_{a'}\bar{A}_{\psi}(s^i, a')\right)$
	\STATE $\tilde{Q^i} \gets r^i\deltat + (1 - d^i) \gamma^{\deltat} V_{\theta}(s'^i)$
	\STATE $\Delta \theta \gets \frac{1}{N}\sum\limits_{i=1}^N  \frac{\left(Q^i - \tilde{Q^i}\right)\partial_{\theta} V_{\theta}(s^i)}{\deltat}$
	\STATE $\Delta \psi \gets \frac{1}{N}\sum\limits_{i=1}^N \frac{\left(Q^i - \tilde{Q^i}\right)\partial_{\psi} \left(\bar{A}_{\psi}(s^i, a^i) - \max\limits_{a'}\bar{A}_{\psi}(s^i, a')\right) }{\deltat}$
	\STATE Update $\theta$ with \textbf{opt}$_1$, $\Delta \theta$ and learning rate $\alpha^V \deltat$.
	\STATE Update $\psi$ with \textbf{opt}$_2$, $\Delta \psi$ and learning rate $\alpha^A \deltat$.
	\ENDFOR
	\ENDFOR
\end{algorithmic}

\end{algorithm}

We learn $V_{\theta}$ and $A_{\psi}$ via suitable variants of $Q$-learning for continuous and
discrete action spaces. Namely, 
the true $A^\pi_\deltat$
and $V^\pi_\deltat$ of a near-continuous MDP with greedy exploitation
policy $\pi(s) \deq
\text{argmax}_{a'}A^\pi_\deltat(s, a')$
are the unique solution to the Bellman and consistency equations
\begin{align}
	V^{\pi}_\deltat(s) + \deltat\, A^{\pi}_\deltat(s, a) &=
	r\,\deltat + \gamma^{\deltat}  \,\E_{s'} V^{\pi}_\deltat(s')\label{eq:bellman_A}\\
	A^{\pi}_\deltat(s, \pi(s)) &= 0\label{eq:max_A}.
\end{align}
as seen in \ref{subsec:reparam}. Thus $V_{\theta}$
and $A_{\psi}$ are trained to approximately solve these equations.

Maximization over actions for $\pi$ is implemented exactly for discrete actions,
and, for continuous actions, approximated by a policy neural network $\pi_\phi(s)$ trained to maximize $A_\psi(s,
\pi_\phi(s))$, similarly to
\cite{ddpg}.

Eq.~\eqref{eq:max_A}
is directly verified by $A_{\psi}$, owing to the reparametrization
$A_\psi(s, a) = \bar{A}_\psi(s, a) - \bar{A}_\psi(s, \pi(s))$, described
in~\ref{subsec:reparam}.  To approximately verify~\eqref{eq:bellman_A}, the
corresponding squared Bellman residual is minimized by an approximate gradient descent.
The update equations when learning from a transition $(s, a, r, s')$, either from
an exploratory trajectory or from a replay buffer~\cite{dqn}, are
\begin{align}
	\delta Q_\deltat &\leftarrow A_\psi(s, a)\, \deltat - \left(r
	\,\deltat + \gamma^{\deltat}\, V_\theta(s') - V_\theta(s)\right)
	\label{eq:approx_deltaQ}\\
	\theta_\deltat &\leftarrow \theta_\deltat + \eta^V_\deltat
	\,\partial_{\theta} V_\theta(s) \,\frac{\delta Q_\deltat}{\deltat}
	\label{eq:approx_bellman_V}\\
	\psi_\deltat &\leftarrow \psi_\deltat + \eta^A_\deltat
	\,\partial_{\psi} A_\psi(s, a) \,\frac{\delta Q_\deltat}{\deltat}.
	\label{eq:approx_bellman_A}
\end{align} 
where the $\eta$'s are learning rates.
Appropriate scalings for the learning rates $\eta^V_\deltat$ and
$\eta^A_\deltat$ in terms of $\deltat$ to obtain a well defined continuous
limit are derived next.

\subsection{Scaling the Learning Rates}
\label{subsec:lr}
\newcommand{\iw}{1cm}
\begin{figure*}[ht]
	\includegraphics[width=\textwidth]{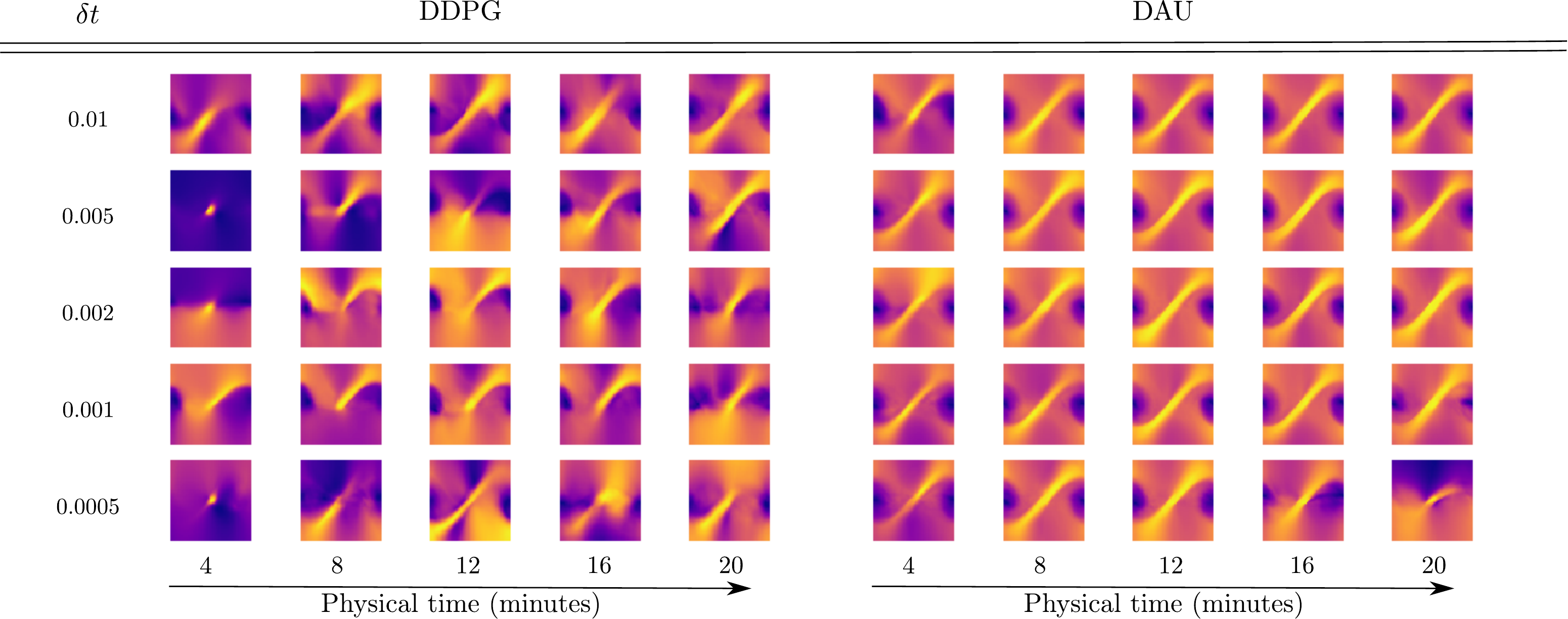}
	\caption{Value functions obtained by DDPG (unscaled version) and
	DAU at different instants in physical time of training on the
	pendulum swing-up environment. Each image represents the learnt
	value function (the $x$-axis is the angle, and the $y$-axis the angular velocity). The lighter the pixel, the higher the value.}
        \label{fig:pend}
\end{figure*}

For the algorithm to admit a continuous-time limit, the discrete-time trajectories
of parameters must converge to well-defined trajectories as $\deltat$ goes to
$0$.
This in turn imposes precise conditions on the scalings of the
learning rates.

Informally, in the parameter updates
\eqref{eq:approx_deltaQ}--\eqref{eq:approx_bellman_A}, the quantity $\delta
Q_\deltat$ is of order $O(\deltat)$, because $s'=s+O(\deltat)$ in a
near-continuous system. Therefore, $\delta
Q_\deltat/\deltat$ is $O(1)$, so that the gradients used to
update $\theta$ and $\psi$ are $O(1)$. Therefore, if the
learning rates are of order $\deltat$, one would expect 
the parameters $\theta$ and $\psi$ to change by $O(\deltat)$ in
each time interval $\deltat$, thus hopefully converging to smooth
continuous-time trajectories. The next theorem formally confirms that
learning rates of order $\deltat$ are the only possibility.



\begin{theorem}
	\label{th:cont-params}
Let $(s_t,a_t)$ be some exploration trajectory in a near-continuous MDP. Set the learning rates to $\eta^V_\deltat =
\alpha^V \deltat^\beta$ and $\eta^A = \alpha^A \deltat^\beta$ for some
$\beta\geq 0$, and learn the parameters $\theta$ and $\psi$ by iterating
\eqref{eq:approx_deltaQ}--\eqref{eq:approx_bellman_A} along the
trajectory $(s_t,a_t)$. Then, when
$\deltat\to 0$:
	\begin{itemize}
		\item If $\beta = 1$ the discrete parameter trajectories converge to continuous parameter
			trajectories.
		\item If $\beta>1$ the parameters stay at their initial
		values.
		\item If $\beta < 1$, the parameters can reach infinity
		in arbitrarily small physical time.
	\end{itemize}
\end{theorem}


The resulting algorithm with suitable scalings,
\emph{Deep Advantage Updating} (DAU), is specified in Alg.~\ref{alg:dau} for
discrete actions, and in the Supplementary for continuous
actions.


\section{Experiments}
\label{sec:exp}
\begin{figure*}[ht]
	\centering
	\includegraphics[width=\textwidth]{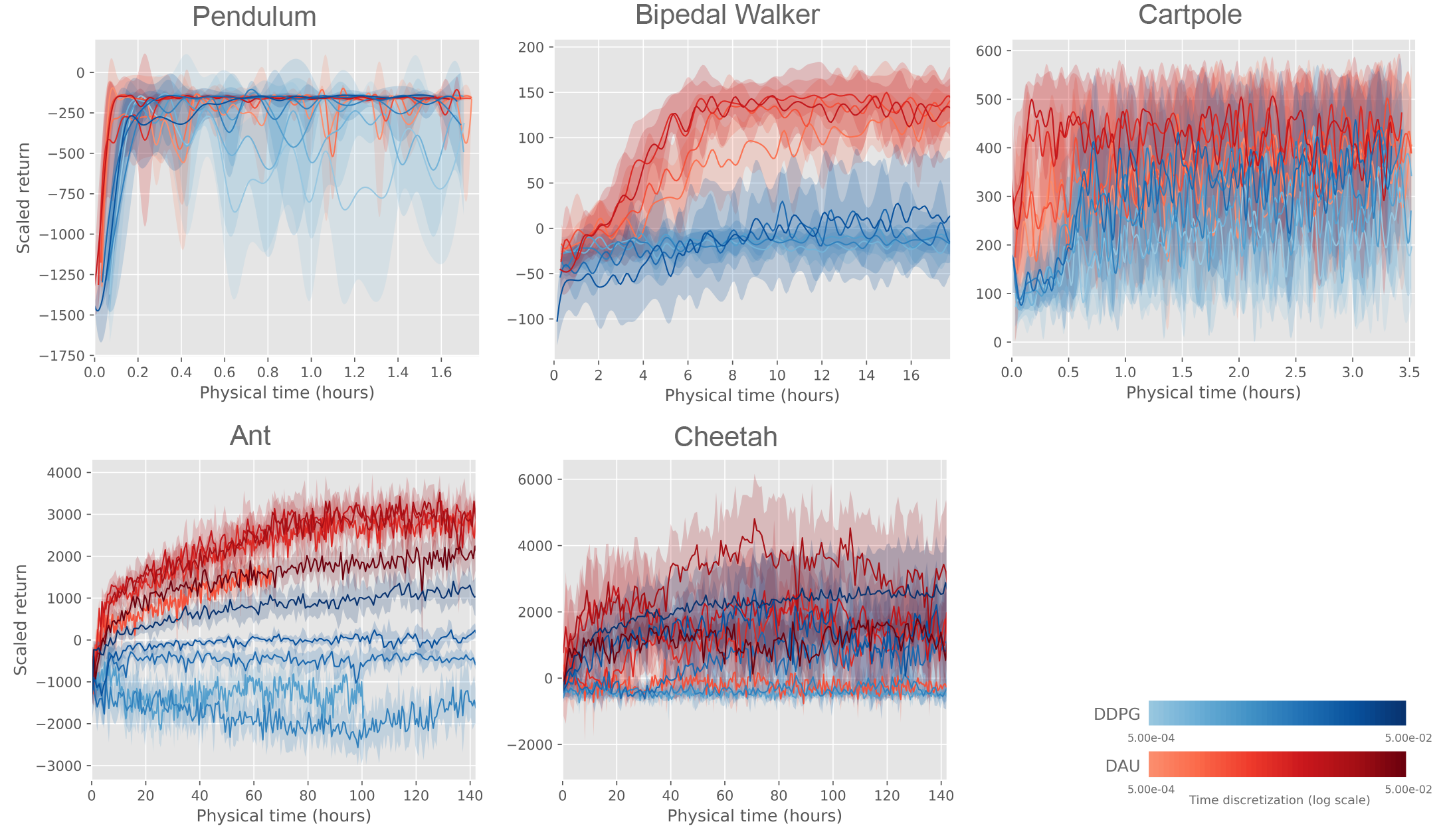}
	\label{fig:full_results}
	\vspace{-2em}
	\caption{Learning curves for DAU and DDPG on classic control benchmarks for  various time discretization $\deltat$: Scaled return as a function of the physical time spent in the environment.}
\end{figure*}

The experiments provided here are specifically aimed at showing that
the proposed method, DAU, works efficiently over a wide range of time
discretizations, without specific tuning, while standard deep $Q$-learning
approaches do not. DAU is compared to DDPG for continuous actions and to DQN
for discrete actions.  As mentionned earlier, we do not study the time
discretization invariance of on-policy methods (A3C, PPO, TRPO...).

In all setups, we use the algorithms described in
Alg.~\ref{alg:dau} and Supplementary Alg.~1. The variants of DDPG and DQN
used are described in the Supplementary, as well as all hyperparameters. We tested two different setups for DDPG and DQN.
In one, we scaled the discount factor (to avoid shortsightedness with small $\deltat$), but
left all other hyperparameters constant across time discretizations.
In the other, we used the properly rescaled discount
factor and reward from Eq.~\eqref{eq:def-gamma},
as well as $O(\deltat)$ learning rates for RMSProp.  The first variant yields slightly
better results, and is presented here, with the second variant in the
Supplementary. For all setups, quantitative results are averaged over five runs.

Let us stress that the quantities plotted are rescaled to make comparison
possible across different timesteps. For example,
returns are given in terms of the discretized return $R_\deltat$ as defined in \eqref{eq:discretized-return},\footnote{This mostly amounts to scaling rewards
by a factor $\deltat$ when this scaling is not naturally done in the environment. Environment-specific
details are given in the Supplementary.} and, most notably, time elapsed is always given in
\emph{physical} time, i.e., the amount of time that the agent spent
interacting with the environment (this is not the number of steps).

\paragraph{Qualitative experiments: Visualizing policies and values.}
To provide qualitative results, and check robustness to time
discretization both in terms of returns and in terms
of convergence of the approximate value function and policies, we first provide results on the simple pendulum environment
from the OpenAI Gym classic control suite.  The state space is of
dimension $2$. We visualize both the learnt value and policy functions by
plotting, for each point of the phase diagram $(\theta, \dot{\theta})$,
its value and policy. The results are presented in
Fig.~\ref{fig:pend} (value function) and Figs.~1, 2,  3 in
Supplementary.

We plot the learnt policy at several instants in physical time
during training, for various time discretizations
$\deltat$, for both DAU and DDPG. With DAU, the agent's policy and value
function quickly converge for every time discretization without specific
tuning. On the contrary, with DDPG, learning of both value function and
policy vary greatly from one discretization to another.

\paragraph{Quantitative experiments.}
We benchmark DAU against DDPG on classic control benchmarks: Pendulum,
Cartpole, BipedalWalker, Ant, and Half-Cheetah environments from OpenAI Gym. On
Pendulum, Bipedal Walker and Ant, DAU is quite robust to variations of
$\deltat$ and displays reasonable performance in all cases. On the other hand,
DDPG's performance varies with $\deltat$; performance either degrades as
$\deltat$ decreases (Ant, Cheetah), or becomes more variable as learning
progresses (Pendulum) for small $\deltat$. On Cartpole, noise dominates,
making interpretation difficult. On Half-Cheetah, DAU is not clearly
invariant to time discretization. This could be explained by the multiple
suboptimal regimes that coexist in the Half-Cheetah environment (walking on the
head, walking on the back), which create discontinuities in the value
function (see Discussion).

%


\section{Discussion}
\label{sec:discussions}

The method derived in this work is theoretically invariant to time
discretization, and
indeed seems to yield improved timestep robustness
on various environments, e.g.,
simple locomotion tasks.  However, on some environments there is still
room for improvment. We discuss some of the issues
that could explain this theoretical/practical discrepancy.

Note that Alg.~\ref{alg:dau} requires
knowledge of the timestep $\deltat$. In most environments, this is
readily available, or even directly set by the practitioner: depending on
the environment it is given by the frame rate, the maximum frequency of
actuators or observation acquisition, the timestep of a physics
simulator, etc. 

\paragraph{Smoothness of the value function.} In our proofs, $V^\pi$ is
assumed to be continuously differentiable.  This 
hypothesis is not always satisfied in practice. 
For instance, in the pendulum swing-up environment, depending on initial
position and momentum, the optimal policy may need to oscillate before
reaching the target state. The optimal value
function is discontinuous at the boundary between states where
oscillations are needed and those where they are not.
This results in non-infinitesimal advantages
for actions on the boundary. In such environments where a given policy
has different regimes depending on the initial state, the continuous-time
analysis only holds almost-everywhere.

\paragraph{Memory buffer size.}~Thm.~\ref{th:cont-params} is stated for
transitions sampled sequentially from a fixed trajectory. In practice,
transitions are sampled from a memory replay buffer, to prevent excessive correlations.
We used a fixed-size circular buffer, filled
with samples from a single growing exploratory trajectory. In
our experiments, the same buffer size was used for all time discretizations. 
Thus the physical-time freshness of samples in the buffer varies with the
time discretization, and in the strictest sense using a fixed-size buffer
breaks timestep invariance. A memory-intensive option would be to use a
buffer of size $\frac{1}{\deltat}$ (fixed memory in physical time).

\paragraph{Near-continuous reinforcement learning and RMSProp.}
RMSProp~\cite{rmsprop} divides gradient steps by the square root of a
moving average estimate of the second moment of gradients.
This may interact with the learning rate scaling discussed above. In
deterministic environments, gradients typically scale as $\bigO(1)$ in
terms of $\deltat$, as seen in \eqref{eq:approx_bellman_A}.  In that case, RMSProp
preconditioning has no effect on the suitable order of magnitude for learning
rates. However, in near continuous \emph{stochastic} environments
(Eq.~\ref{eq:sde}), variance of $\delta
Q_{\deltat}/\deltat$ and of the gradients typically scales as
$\bigO\left(1/\deltat\right)$. With a fixed batch size,
RMSProp will multiply gradients by a factor $\bigO(\sqrt{\deltat})$. In
that case,
learning rates need only be scaled as $\sqrt{\deltat}$ instead of
$\deltat$.

More generally, the continuous-time analysis should in principle be repeated for every component of  
a system. For instance, if a recurrent model is used to handle state
memory or partial observability, care should be taken that the model is
able to maintain memory for a non-infinitesimal physical time when
$\deltat\to 0$ (see e.g.~\citealt{chronornn}).

\section{Conclusion}
$Q$-learning methods have been found to fail to learn with small time
steps, both theoretically and empirically. A theoretical analysis
help in building a practical off-policy deep RL algorithm with better robustness to time
discretization. This robustness is confirmed empirically.
%

\section*{Acknowledgments}
We would like to thank Harsh Satija and Joelle Pineau for their useful remarks and comments.
\bibliography{icml_drau.bib}
\bibliographystyle{icml2019}
\vfill

\pagebreak
\appendix



\section{Proofs}

We now give proofs for all the results presented in the paper. Most
proofs follow standard patterns from calculus and numerical schemes for
differential equations, except for Theorem~\ref{thm:policyimprovement},
which uses an argument specific to reinforcement learning to prove that
the continuous-time advantage function contains all the necessary
information for policy improvement.

The first result presented is a proof of convergence for discretized
trajectories.
\begin{lemma}
	Let $F\from {\cal S} \times {\cal A} \rightarrow \bb{R}^n$ and $\pi\from {\cal S}
	\rightarrow {\cal A}$ be the dynamic and policy functions. Assume that,
	for any $a$, $s \rightarrow F(s, a)$ and $s \rightarrow F(s, \pi(s))$
	are ${\cal C}^1$, bounded and $K$-lipschitz.  For
	a given $s_0$, define the trajectory $(s_t)_{t\geq 0}$ as the unique
	solution of the differential equation
	\begin{equation}
		\frac{ds_t}{dt} = F(s_t, \pi(s_t)).
		\label{eq:diff}
	\end{equation}
	For any $\deltat > 0$, define the discretized trajectory
	$(s_\deltat^k)_k$ which amounts to maintaining each action for a
	time interval $\deltat$; it is defined by induction as $s_\deltat^0 = s_0$,
	$s_\deltat^{k + 1}$ is the value at time $\deltat$ of
	the unique solution of
	\begin{equation}
		\frac{d\tilde{s}_t}{dt} = F(\tilde{s}_t, \pi(s_\deltat^k))
		\label{eq:diff_2}
	\end{equation}
	with initial point $s_\deltat^k$.
	Then, there exists $C > 0$ such that, for every $t \geq 0$
	\begin{equation}
		\|s_t - s_\deltat^{\lfloor \frac{t}{\deltat} \rfloor}\|
		\leq \deltat \frac{C}{K}e^{Kt}.
	\end{equation}
	Therefore, discretized trajectories converge pointwise to continuous trajectories.
	\label{th:traj-conv}
\end{lemma}
\begin{proof}
The proof mostly follos the classical argument for convergence of the
Euler scheme for differential equations.
For any $k$, define
\begin{equation}
	e_\deltat^k = \|s_\deltat^k - s_{\deltat k}\|.
\end{equation}
Let $\tilde s_t$ be the solution of Eq.~\eqref{eq:diff_2} with initial state $s^k_\delta$. This $\tilde s_t$ is ${\cal
C}^2$ on $[0, \deltat]$. Consequently, the Taylor integral formula gives
\begin{align}
	s_\deltat^{k + 1} &= s_\deltat^k + F(s_\deltat^k, \pi(s_\deltat^k)) \deltat +
	\int_{0}^\deltat (\deltat - t) \frac{d^2 \tilde{s}_t}{dt^2} dt\\
	s_{\deltat(k + 1)} &= s_{\deltat k} + F(s_{\deltat k}, \pi(s_{\deltat k})) \deltat +
	\int_0^\deltat (\deltat - t) \frac{d^2 s_{t + \deltat k}}{dt^2} dt.
\end{align}
Now, both $d^2 s_t/dt^2$ and $d^2 \tilde{s}_t/dt^2$ are uniformly bounded, by
boundedness and Lipschitzness of $s \rightarrow F(s, \pi(s))$ and $s
\rightarrow F(s, \pi(s_\deltat^k))$. Consequently, there exists $C$ such that
\begin{align}
	e_\deltat^{k + 1} &\leq \|s_\deltat^k - s_{\deltat k}\| + \|F(s_\deltat^k, \pi(s_\deltat^k)) - F(s_{\deltat k}, \pi(s_{\deltat k}))\| \deltat + C \deltat^2\\
			  &\leq (1 + K \deltat) e_\deltat^k + C \deltat^2.
\end{align}
Now, it is easy to prove by induction that
\begin{equation}
	e_\deltat^k \leq (1 + K \deltat)^k (e_\deltat^0 + \frac{C}{K} \deltat) - \frac{C}{K}\deltat.
\end{equation}
As $e_\deltat^0 = 0$, this translates to
\begin{align}
	e_\deltat^k &\leq ((1 + K \deltat)^k - 1) \deltat \frac{C}{K}\\
		    &\leq (e^{K \deltat k} - 1) \deltat \frac{C}{K}.
\end{align}
Consequently,
\begin{equation}
	e_\deltat^{\lfloor t / \deltat \rfloor} \leq (e^{K(t + \deltat)} - 1) \deltat \frac{C}{K}.
	\label{eq:bound_err}
\end{equation}
Finally, by boundedness, of $s \rightarrow F(s, \pi(s))$, there exists $C'$ such that
\begin{equation}
	\|s_{\deltat\lfloor t / \deltat \rfloor} - s_t \| \leq \deltat C'.
	\label{eq:close_s}
\end{equation}
Combining Eq.~\eqref{eq:close_s} with Eq.~\eqref{eq:bound_err}, one can find $C''$ such that
\begin{equation}
	\|s_t - s_\deltat^{\lfloor t / \deltat \rfloor}\| \leq \deltat \frac{C''}{K} e^{Kt}.
\end{equation}

\end{proof}

In what follows, we assume that the continuous-time reward function $r\from {\cal S} \times {\cal A} \rightarrow \bb{R}$
is bounded, to ensure existence of $V^\pi$ and $V^\pi_\deltat$ for all $\deltat$
.\begin{theorem}
	Assume that $r\from {\cal S} \times {\cal A} \rightarrow \bb R$ is bounded, and
	that $s \rightarrow r(s, \pi(s))$ is $L_r$-Lipschitz continuous, then
	for all $s \in {\cal S}$, one has
	$V^\pi_\deltat(s) = V^\pi(s) + \smallo(1)$
	when $\deltat\to 0$.
	\label{th:conv-value}
\end{theorem}
\begin{proof}

We use the notation $\tilde r(s) = r(s, \pi(s))$. Let $\tilde s_\deltat^t \deq s_\deltat^{\lfloor t/\deltat\rfloor}$. We have:
\begin{align}
  V^\pi_\deltat(s) = \int_t \gamma^t \tilde r(\tilde s_\deltat^t)dt + \bigO(\deltat)
\end{align}
Indeed:
\begin{align}
  V^\pi_\deltat(s) =& \sum_{k=0}^\infty \gamma^{k\deltat}\tilde r(s_\deltat^k)\deltat \\
  =& \sum_{k=0}^\infty \gamma^{k\deltat}\int_{u=k}^{k+1}\tilde r(\tilde s_\deltat^{u\deltat})du \\
  =& \sum_{k=0}^\infty \frac{\deltat \log \gamma}{\gamma^\deltat - 1}\int_{u=k}^{k+1}\gamma^{u\deltat}\tilde r(\tilde s_\deltat^{u\deltat}) du \\
  =& \frac{\deltat \log \gamma}{\gamma^\deltat - 1}  \int_{t=0}^{\infty}\gamma^{t}\tilde r(\tilde s_\deltat^{t})dt \\
\end{align}
But:
\begin{align}
  \frac{\deltat \log \gamma}{\gamma^\deltat - 1} &= \frac{\deltat \log \gamma}{\deltat \log\gamma + \bigO(\deltat^2)} \\
  &= 1 +  \bigO(\deltat) \\
\end{align}

Therefore:
\begin{align}
  V^\pi_\deltat(s) = \int_t \gamma^t \tilde r(\tilde s_\deltat^t)dt + \bigO(\deltat)
\end{align}

We now have, for any $T>0$,
\begin{align}
  |V^\pi_\deltat(s) - V^\pi(s)|  &= | \int_{t=0}^\infty\gamma^t \left(\tilde r(\tilde s_\deltat^t) - \tilde r(s_t) \right)dt | + \bigO(\deltat) \\
				 &= | \int_{t=0}^T\gamma^t \left(\tilde r(\tilde s_\deltat^t) - r(s_t) \right)dt | \\
				 &+ | \int_{t=T}^\infty\gamma^t \left(\tilde r(\tilde s_\deltat^t) - \tilde r(s_t) \right)dt | + \bigO(\deltat)
\end{align}

The second term can be bounded by the supremum of the reward:
\begin{align}
  | \int_{t=T}^\infty\gamma^t \left(\tilde r(\tilde s_\deltat^t) - \tilde r(\tilde s_t) \right)dt | \leq 2 \frac{\|r\|_\infty}{\log(\frac{1}{\gamma})} \gamma^T
  \label{eq:v-firstpart}
\end{align}

The first term can be bounded by using Lemma.~1:
\begin{align}
  | \int_{t=0}^T\gamma^t &\left(\tilde r(\tilde s_\deltat^t) - \tilde r(s_t) \right)dt | \leq \int_{t=0}^T\gamma^t L_r \| s_t - \tilde s_\deltat^t \| dt \\
  \leq& \int_{t=0}^T L_r \frac{C \deltat}{K} \exp((K + \log \gamma)t ) dt \\
  \leq&  \frac{L_r C}{K(K + \log \gamma)}\exp((K + \log \gamma)T ) \deltat
   \label{eq:v-secondpart}
\end{align}
Let us set $T \deq -\frac{1}{K}\log(\deltat)$. By plugging into Eq. (\ref{eq:v-firstpart}), we have:
\begin{align}
	| \int_{t=T}^\infty\gamma^t \left(\tilde r(\tilde s_\deltat^t) - \tilde r(s_t) \right)dt | = \bigO( \deltat^{-\log\gamma}) = \smallo(1).
\end{align}
By plugging $T$ into equation (\ref{eq:v-secondpart}), we have:
\begin{align}
	| \int_{t=0}^T\gamma^t \left(\tilde r(\tilde s_\deltat^t) - \tilde r(s_t) \right)dt | = \bigO( \deltat^{-\frac{\log\gamma}{K}}) = \smallo(1),
\end{align}
yielding our result.

\end{proof}

For the following proof, we further assume that both $V^\pi$ and
$V^\pi_\deltat$ are continuously differentiable, and that the gradient and
Hessian of $V^\pi_\deltat$ w.r.t. $s$ are uniformly bounded in both $s$ and $\deltat$.
We also assume convergence of $\partial_s V^\pi_\deltat(s)$ to $\partial_s V^\pi(s)$ for
all $s$.

\begin{theorem}
	Under the hypothesis above, there exists $A^\pi\from {\cal S} \rightarrow
	\bb{R}$ such that $A^\pi_\deltat$ converges pointwise to $A^\pi$ as
	$\deltat$ goes to $0$. Besides,
	\begin{equation}
		A^\pi(s, a) = r(s, a) + \partial_s V^\pi(s) F(s, a) + \log \gamma V^\pi(s).
	\end{equation}\end{theorem}
\begin{proof}
Denote $\tilde{s}_\deltat^t(s_0)$ the
evaluation at instant $t$ of the solution of $d\tilde{s}_t/dt = F(\tilde{s}_t,
\pi(s_0))$ with starting point $s_0$.

The Bellman equation on $Q^\pi_\deltat$ yields
\begin{align}
	Q^\pi_\deltat(s, a) &= r(s, a) \deltat + \gamma^\deltat V^\pi_\deltat(\tilde{s}_\deltat^\deltat(s)).
\end{align}

For all $s$, a first-order Taylor expansion yields
\begin{equation}
	\tilde{s}_\deltat^\deltat(s) = s + F(s, a)\deltat + \bigO(\deltat^2)
\end{equation}
where the constant in $\bigO()$ is uniformly bounded thanks to the
assumptions on the Hessian.
Thus, by uniform boundedness of the Hessian of $V^\pi_\deltat$,\footnote{
	Without boundedness of the Hessian, we cannot write the second order Taylor expansion of
	$V^\pi_\deltat(\tilde{s}_\deltat^\deltat(s))$ in term of $\deltat$.
}
\begin{align}
  &Q^\pi_\deltat(s, a) = r(s, a) \deltat + \\
  &\left(1 + \ln(\gamma) \deltat + \bigO(\deltat^2)\right) \left(V^\pi_\deltat(s) + \deltat \partial_s V^\pi_\deltat(s) F(s, a) + \bigO(\deltat^2)\right).\nonumber
\end{align}
Now, this yields
\begin{equation}
	A^\pi_\deltat(s, a) = r(s, a) + \ln(\gamma) V^\pi_\deltat(s) + \partial_s V^\pi_\deltat(s) F(s, a) + \bigO(\deltat),
\end{equation}
and using the convergence of $V^\pi_\deltat(s)$ to $V^\pi(s)$ (Thm.~\ref{th:conv-value}) and $\partial_s V^\pi_\deltat(s)$ to $\partial_s V^\pi(s)$ (hypothesis) yields the result with 
\begin{equation}
	A^\pi(s, a) = r(s, a) + \ln(\gamma) V^\pi(s) + \partial_s V^\pi(s) F(s, a).
	\label{eq:adv_function}
\end{equation}

\end{proof}

We now show that policy improvement works with the continous time advantage function, i.e.\
\begin{theorem}
\label{thm:policyimprovement}
	Let $\pi$ and $\pi'$ be two policies such that both $s \rightarrow r(s, \pi(s))$ and
	$s \rightarrow r(s, \pi'(s))$ are continuous.
	Assume that both $V^\pi$ and $V^{\pi'}$ are continuously differentiable.
	Define the advantage function for policies $\pi$ and $\pi'$ as in Eq.~\eqref{eq:adv_function}.

	If for all $s$, $A^\pi(s, \pi'(s)) \geq 0$, then for all $s$,
	$V^\pi(s) \leq V^{\pi'}(s)$. Moreover, if for all $s$, 
	$V^{\pi'}(s)>V^{\pi}(s)$, then there exists $s'$ such that $A^\pi(s',\pi'(s'))>0$.
\end{theorem}

\begin{proof}
Let $(s_t)_{t\geq 0}$ be a trajectory sampled from $\pi'$ i.e.\ solution of the equation
\begin{equation}
	ds_t / dt = F(s_t, \pi'(s_t))
\end{equation}
with initial condition $s_0 = s$.

Define
\begin{equation}
	B(T) = \int_{t=0}^T \gamma^t r(s_t, \pi'(s_t)) dt + \gamma^T V^\pi(s_T).
\end{equation}
This function if continuously differentiable, and its derivative is
\begin{align}
	\dot{B}(T) &= \gamma^T r(s_T, \pi'(s_T)) \\ &+ \gamma^T \partial_s V^\pi(s) F(s, \pi'(s)) + \gamma^T \ln(\gamma) V^\pi(s_T)\\
		   &= \gamma^T A^\pi(s_T, \pi'(s_T))\\
		   &\geq 0.
\end{align}
Thus $B$ is increasing, and $B(0) = V^\pi(s)$, $\lim\limits_{T\rightarrow \infty} B(t) = V^{\pi'}(s)$.
Consequently, $V^\pi(s) \leq V^{\pi'}(s)$.
Furthermore, if $V^\pi(s) < V^{\pi'}(s)$, then there exists $T_0$ such that $\dot B(T_0) > 0$ (otherwise $B$ is constant), and $A^\pi(s_{T_0}, \pi'(s_{T_0})) > 0$.

\end{proof}
\begin{theorem}
	Let $\cal A = \bb R^A$ be the action space, and let ${\cal P}_1 = \bb
	R^{p_1}$ and ${\cal P}_2 = \bb R^{p_2}$ be parameter spaces.  Let
	$A\colon {\cal P}_1 \times {\cal S} \times {\cal A} \rightarrow \bb R$
	and $V\colon {\cal P}_2 \times {\cal S} \rightarrow \bb R$ be ${\cal
	C}^2$ function approximators with bounded gradients and Hessians. Let
	$(a_t)_{t\geq 0}$ be a ${\cal C}^1$ exploratory action trajectory and
	$(s_t)_{t\geq 0}$ the resulting state trajectory, when starting from $s_0$ and
	following $ds_t/dt=F(s_t,a_t)$.  Let $\theta^k_\deltat$ and
	$\psi^k_\deltat$ be the discrete parameter trajectories resulting from
	the gradient descent steps in the main text, with
	learning rates $\eta^V = \alpha^V \deltat^\beta$ and $\eta^A =
	\alpha^A \deltat^\beta$ for some $\beta\geq 0$. Then,
	\begin{itemize}
		\item If $\beta = 1$ the discrete parameter trajectories converge to continuous parameter
			trajectories as $\deltat$ goes to $0$.
		\item If $\beta > 1$, parameter trajectories become
		stationary as
			$\deltat$ goes to $0$.
		\item If $\beta < 1$, parameters can grow arbitrarily large after an arbitrarily small physical time when $\deltat$ goes to $0$.
	\end{itemize}
\end{theorem}
\begin{proof}
Let $(s_t, a_t)_{t\geq 0}$ be the trajectory on which parameters are
learnt. To simplify notations, define
\begin{equation}
	A_\psi(s, a) = \bar{A}_\psi(s, a) - \bar{A}_\psi(s, \pi(s)).
\end{equation}
Define $F$ as
\begin{align}
	F^\theta(\theta, \psi, s, a) &= \alpha^V (r(s, a) + \ln(\gamma) V_\theta(s) \\ \nonumber &+ \partial_s V_\theta(s) F(s, a) - A_\psi(s, a))\partial_\theta V_\theta(s)\\
	F^\psi(\theta, \psi, s, a) &= \alpha^A (r(s, a) + \ln(\gamma) V_\theta(s) \\ \nonumber &+ \partial_s V_\theta(s) F(s, a) - A_\psi(s, a))\partial_\psi A_\psi(s, a).
\end{align}
From the bounded Hessians and Gradients hypothesis, $V$, $A$, $\partial_s V$,
$\partial_\theta V$ and $\partial_\psi A$ are uniformly Lipschitz continuous in
$\theta$ and $\psi$, thus $F$ is Lipschitz continuous.

The discrete equations for parameters updates with learning rates $\alpha^V \deltat^\beta$ and
$\alpha^A \deltat^\beta$ are
\begin{align}
	\delta Q &= r(s_{k \deltat}, a_{k \deltat}) \deltat + 
		\gamma^\deltat V_{\theta^{k}_\deltat}(s_{(k + 1)\deltat}) \\ \nonumber &- 
		V_{\theta^k_\deltat}(s_{k\deltat}) - A_\psi(s_{k \deltat}, a_{k \deltat})\\
	\theta^{k+1}_\deltat &= \theta^k_\deltat + 
	\alpha^V \deltat^\beta \frac{\delta Q
	}{\deltat}\partial_\theta V_{\theta^k_\deltat}(s_{k\deltat})\\
	\psi^{k+1}_\deltat &= \psi^k_\deltat + 
	\alpha^A \deltat^\beta \frac{\delta Q
	}{\deltat}
	\partial_\psi A_{\theta^k_\deltat}(s_{k\deltat}, a_{k\deltat})
\end{align}
Under uniform boundedness of the Hessian of $s \mapsto V_\theta(s)$,
one can show
\begin{equation}
	\begin{pmatrix}
		\theta_\deltat^{k+1} \\
		\psi_\deltat^{k+1}
	\end{pmatrix} =
	\begin{pmatrix}
		\theta_\deltat^{k} \\
		\psi_\deltat^{k}
	\end{pmatrix} + \deltat^\beta F(\theta_\deltat^k, \psi_\deltat^k, s_{k \deltat}, a_{k\deltat}) + \bigO(\deltat^\beta \deltat),
\end{equation}
with a $\bigO$ independent of $k$. With the additional hypothesis that the gradient of
$(s, a) \rightarrow \bar{A}_\psi(s, a)$ is uniformly bounded, we have
\begin{itemize}
	\item For $\beta = 1$, a proof scheme identical to that of Thm.~\ref{th:traj-conv} shows that
discrete trajectories converge pointwise to continuous trajectories defined by the differential equation
\begin{equation}
	\frac{d}{dt}\begin{pmatrix}
		\theta_t \\
		\psi_t
	\end{pmatrix} =
	F(\theta_t, \psi_t, s_t, a_t),
\end{equation}
which admits unique solutions for all initial parameters, since $F$ is uniformly lipschitz continuous.
\item Similarly, for $\beta > 1$, the proof scheme of Thm.~\ref{th:traj-conv} shows that
discrete trajectories converge pointwise to continuous trajectories defined by the differential equation
\begin{equation}
	\frac{d}{dt}\begin{pmatrix}
		\theta_t \\
		\psi_t
	\end{pmatrix} = 0
\end{equation}
and thus that trajectories shrink to a single point as $\deltat$ goes to $0$.
\end{itemize}

We now turn to proving that when $\beta < 1$,
trajectories can diverge instantly in physical time. 
Consider the following continuous MDP, 
\begin{equation}
	s_t = \sin(t)
\end{equation}
whatever the actions,
with reward $0$ everywhere and $0 < \gamma < 1$.
The resulting value function is $V(s) = 0$ (since there
are no actions, $V$ is independent of a policy), and the advantage
function is $0$.
We consider the function approximator $V_\theta(s) = \theta s$ (which can
represent the true value function).
The update rule for $\theta$ is
\begin{align}
\delta Q_\deltat^k &= \gamma^\deltat \theta_\deltat^k \sin((k + 1)
\deltat) - \theta_\deltat^k \sin(k\deltat)\\
	\theta_\deltat^{k+1} &= \theta_\deltat^k + \alpha \deltat^\beta
	\frac{\gamma^\deltat \theta_\deltat^k \sin((k + 1) \deltat) -
	\theta_\deltat^k \sin(k\deltat)}{\deltat}\,\sin(k\deltat)
\end{align}
Set $K_\deltat \deq \lfloor \deltat^{-\frac{\beta + 3}{4}}\rfloor$, then for all
$k \leq K_\deltat$, $\smallo(k\deltat) = \smallo(1)$ and
\begin{align}
	\theta_\deltat^{k+1} &= \theta_\deltat^k(1 + \alpha \deltat^\beta
	(1 + \smallo(1))\sin(k\deltat))\\
\end{align}
Let $\rho_\deltat^k \deq \log \theta_\deltat^k$. Then
\begin{align}
	\rho_\deltat^k = \rho_\deltat^k + \alpha
	k\deltat^{\beta + 1} + \smallo(k\deltat^{\beta + 1}).
\end{align}
Finally,
\begin{align}
	\rho_\deltat^{K_\deltat} &= \rho_\deltat^0 + \alpha \frac{K_\deltat (K_\deltat + 1)}{2}\deltat^\beta + \smallo(K_\deltat^2\deltat^{\beta + 1})\\
				 &= \rho_\delta^0 + \alpha\deltat^\frac{\beta
					 - 1}{3} + \smallo(\deltat^\frac{\beta - 1}{3})\\
				 &\xrightarrow[\deltat\to 0]{} +\infty.
\end{align}
Thus parameters diverge in an infinitesimal physical time when $\deltat$ goes to $0$.

\end{proof}

\begin{theorem}
  Let $F\from {\cal S} \times {\cal A} \rightarrow \bb{R}^n$ be the dynamic, and $\pi\from {\cal S} \times {\cal A} \rightarrow [0,1]$ be the policy, such that $\pi(s, \cdot)$ is a probability distribution over ${\cal A}$. Assume that 
  $F$ is $C^1$ with bounded derivatives, and that $\pi$ is $C^1$ and bounded.
  For any $\deltat > 0$, define the discretized trajectory
	$(s_\deltat^k)_k$ which amounts to sample an action from $\pi(s, \cdot)$ and maintaining each action for a
	time interval $\deltat$; it is defined by induction as $s_\deltat^0 = s_0$,
	$s_\deltat^{k + 1}$ is the value at time $\deltat$ of
	the unique solution of
	\begin{equation}
		\frac{d\tilde{s}_t}{dt} = F(\tilde{s}_t, a_k)
		\label{eq:diff_2}
	\end{equation}
	with $a_k \sim \pi(s_\deltat^k, \cdot)$ and initial point $s_\deltat^k$.

        Then the
        agent's trajectories converge when $\deltat\to 0$ to the solutions of the
        \emph{deterministic} equation:
        \begin{equation}
		\frac{ds_t}{dt} = \E_{a\sim \pi(s_t, \cdot)}F(s_t, a).
		\label{eq:diff}
              \end{equation}
              Notably, if $\pi$ is an epsilon greedy strategy that mixes a deterministic exploitation policy $\pi^{\text{deterministic}}$ with an action taken from a noise policy $\pi^{\text{noise}}$ with probability $\varepsilon$ at, the trajectory converge to the solutions of the equation: 
        \begin{equation}
          d s_t/dt= (1-\eps) F(s_t,\pi^\text{deterministic}(s_t))+\eps \E_{a\sim
            \pi^\text{noise}(a|s)}F(s_t,a)
	\end{equation}
\end{theorem}
\begin{proof}

Consider $(s_{\deltat^2})$ the random trajectory of a near-continuous MDP with time-discretization $\deltat^2$ obtained by taking at each step $k$ an action $a_k$ along $a_k \sim \pi(a| s_{\deltat^2}^k)$ independantly. We have:

\begin{align}
  s_{\deltat^2}^{\lfloor 1/\deltat\rfloor} &= s_{\deltat^2}^0 + \sum_{k=1}^{\lfloor 1/\deltat\rfloor} s_{\deltat^2}^{k} -  s_{\deltat^2}^{k-1} + \bigO(\deltat^2) \\
                    &= s_{\deltat^2}^0 + \sum_{k=1}^{\lfloor 1/\deltat \rfloor} F(s_{\deltat^2}^{k-1}, a_{k-1})\deltat^2 + \bigO(\deltat^2)
\end{align}

We define $f(s) \deq \E_{a \sim \pi(s)}\left[F(s, a)\right] = \int_{a\in{\cal A}}F(s, a)\pi(s, a)$. Since $\pi$ and $F$ are bounded and $C^1$, we know that $f$ is $C^1$. We have: 
  \begin{align}
    s_{\deltat^2}^{\lfloor 1/\deltat\rfloor}   & = s_{\deltat^2}^0 + \sum_{k=1}^{\lfloor 1/\deltat\rfloor}
                        f(s_{\deltat^2}^{k-1})\deltat^2 \\ \nonumber &+ \sum_{k=1}^{\lfloor 1/\deltat\rfloor} (F(s_{\deltat^2}^{k-1}, a_{k-1}) - f(s_{\deltat^2}^{k-1}))\deltat^2 + \bigO(\deltat^2) \\
     s_{\deltat^2}^{\lfloor 1/\deltat\rfloor}  &= s_{\deltat^2}^0 + \sum_{k=1}^{\lfloor 1/\deltat\rfloor}
                        f(s_{\deltat^2}^{k-1})\deltat^2 + \xi + \bigO(\deltat^2)
  \end{align}
  with $\xi \deq  \deltat^2 \sum_{k=1}^{\lfloor 1/\deltat\rfloor} \left(F(s_{\deltat^2}^{k-1}, a_{k-1}) - f(s_{\deltat^2}^{k-1})\right)$. By definition, we have $\E[\xi] = 0$. Moreover, by using the independance of actions and the boundness of F, there is $\sigma >0$ such that:
  \begin{align}
    \E[\|\xi\|^2] \leq \sigma^2\deltat^3
  \end{align}

  We know that $f$ is $C^1$ on a compact space. Therefore, there is $L_f$ such that $f$ is $L_f$ Lipschitz, and we have:
  \begin{align}
    \|\left(\sum_{k=1}^{\lfloor 1/\deltat\rfloor} f(s_{\deltat^2}^{k-1})\deltat\right) - f(s_{\deltat^2}^0) \| \leq \deltat L_f\sum_{k=1}^{\lfloor 1/\deltat\rfloor} \|s_{\deltat^2}^{k-1} - s_{\deltat^2}^0\|
  \end{align}
  Since $F$ is bounded, we know that $\|s_{\deltat^2}^{k} - s_{\deltat^2}^{k-1}\| \leq C\deltat$. Therefore:
  \begin{align}
    \|\left(\sum_{k=1}^{\lfloor 1/\deltat\rfloor} f(s_{\deltat^2}^{k-1})\deltat\right) - f(s_{\deltat^2}^0) \| &\leq \deltat L_fC\sum_{k=1}^{\lfloor 1/\deltat\rfloor} k\deltat \\
    &= \bigO(\deltat^2) 
  \end{align}
  Therefore: 
  \begin{align}
  s_{\deltat^2}^{\lfloor 1/\deltat \rfloor} =  s_{\deltat^2}^0 + f(s_{\deltat^2}^0) \deltat + \xi + \bigO(\deltat^2) 
  \end{align}

  Therefore, we have $a>0$ such that $\|s_{\deltat^2}^{\lfloor 1/\deltat \rfloor} -  s_{\deltat^2}^0 - f(s_{\deltat^2}^0) \deltat\| \leq \|\xi\| + a\deltat^2$
  
We define $(\tilde s_\deltat)$ the deterministic near-continuous process with time discretization $\deltat$ defined by $\tilde s_\deltat^{k+1} \deq s_\deltat^k + f(s_\deltat^k)\deltat$. We have:
\begin{align}
  \|s_{\deltat^2}^{(k+1)\lfloor 1/\deltat\rfloor}  &- \tilde{s}_\deltat^{k+1}\| \nonumber \\&\leq \|s_{\deltat^2}^{(k+1)\lfloor 1/\deltat\rfloor} - s_{\deltat^2}^{k\lfloor 1/\deltat\rfloor} - f(s_{\deltat^2}^{k\lfloor 1/\deltat\rfloor})\deltat\| \nonumber \\&+ \|s_{\deltat^2}^{k\lfloor 1/\deltat\rfloor} + f(s_{\deltat^2}^{k\lfloor 1/\deltat\rfloor})\deltat - \tilde{s}_\deltat^{k+1}\|
\end{align}
We know that
\begin{equation}
\|s_{\deltat^2}^{(k+1)\lfloor 1/\deltat\rfloor} - s_{\deltat^2}^{k\lfloor 1/\deltat\rfloor} - f(s_{\deltat^2}^{k\lfloor 1/\deltat\rfloor})\deltat\| \leq \|\xi_k\| + a\deltat^2
\end{equation}
Moreover:
\begin{align}
  \nonumber \|s_{\deltat^2}^{k\lfloor 1/\deltat\rfloor} &+ f(s_{\deltat^2}^{k\lfloor 1/\deltat\rfloor})\deltat - \tilde{s}_\deltat^{k+1}\|  \\ &\leq \|s_{\deltat^2}^{k\lfloor 1/\deltat\rfloor} - \tilde{s}_\deltat^{k} \| + \deltat\|f(s_{\deltat^2}^{k\lfloor 1/\deltat\rfloor}) - f(\tilde{s}_\deltat^{k})\| \\
  &\leq (1+L_f\deltat)\|s_{\deltat^2}^{k\lfloor 1/\deltat\rfloor} - \tilde{s}_\deltat^{k} \| 
\end{align}
Therefore, we have:
\begin{align}
 \|s_{\deltat^2}^{(k+1)\lfloor 1/\deltat\rfloor}  &- \tilde{s}_\deltat^{k+1}\| \nonumber \\ &\leq \|\xi_k\| + a\deltat^2 + (1+L_f\deltat)\|s_{\deltat^2}^{k\lfloor 1/\deltat\rfloor}  - \tilde{s}_\deltat^{k}\|  
\end{align}
By induction, and by taking $k = \lfloor t/\deltat\rfloor$:
\begin{align}
  \|s_{\deltat^2}^{k\lfloor 1/\deltat\rfloor}  - \tilde{s}_\deltat^{k}\| \leq \frac{a\deltat}{L_f}\exp(L_ft) + \sum_{j=0}^{\lfloor t/\deltat \rfloor}(1+\deltat L_f)^j\|\xi_j\|  
\end{align}

Therefore, if $\varepsilon >0$, we have :
\begin{align}
  \BP&\left(\|s_{\deltat^2}^{k\lfloor 1/\deltat\rfloor}  - \tilde{s}_\deltat^{k}\| > \varepsilon \right) \\ &\leq \BP\left(\sum_{j=0}^{\lfloor t/\deltat \rfloor}(1+\deltat L_f)^j\|\xi_j\| > \varepsilon - \frac{a\deltat}{L_f}\exp(L_ft)\right) \\
                                                                                            &\leq \frac{\E\left[\sum_{j=0}^{\lfloor t/\deltat \rfloor}(1+\deltat L_f)^j\|\xi_j\|\right]}{\varepsilon - \frac{a\deltat}{L_f}\exp(L_ft)}\\
                                                                                            &\leq \frac{\E\left[\|\xi\|\right]}{\varepsilon - \frac{a\deltat}{L_f}\exp(L_ft)}\frac{\exp(L_ft)}{L_f\deltat}\\
\end{align}
But $\E\left[\|\xi\|\right] \leq \sqrt{\E\left[\|\xi\|^2\right]} \leq \sigma \deltat^{3/2}$. Therefore, we have:
\begin{align}
  \BP\left(\|s_{\deltat^2}^{\lfloor t/\deltat\rfloor\lfloor 1/\deltat\rfloor}  - \tilde{s}_\deltat^{k}\| > \varepsilon \right) & = \bigO(\sqrt \deltat)
\end{align}

Therefore, the  process $t \mapsto s_{\deltat^2}^{\lfloor t/\deltat\rfloor\lfloor 1/\deltat\rfloor}$ converges in probability to $\tilde s$. Furthermore, by a similar argument than in Lemma 1, we know that the discretized process $\tilde s$ converge to the continuous process defined by $\frac{ds}{dt} = f(s_t)$. We can conclude ou result.


\end{proof}
\section{Implementation details}
All the details specifying our implementation are given in this section. We
first give precise pseudo code descriptions for both \emph{Continuous Deep
Advantage Updating} (Alg.~\ref{alg:cdau}), as well as the variants of DDPG
(Alg.~\ref{alg:ddpg}) and DQN (Alg.~\ref{alg:dqn}) used.

\begin{algorithm}

\begin{algorithmic}
	\STATE \textbf{Inputs:}
	\STATE $\theta$, $\psi$ and $\phi$, parameters of
	$V_{\theta}$, $\bar{A}_{\psi}$ and $\pi_\phi$.
	\STATE $\pi^{\text{explore}}$ and $\nu_\deltat$ defining an exploration policy.
	\STATE \textbf{opt}$_V$, \textbf{opt}$_A$, \textbf{opt}$_\pi$, $\alpha^V \deltat$, $\alpha^A \deltat$ and $\alpha^\pi\deltat$, optimizers and learning rates.
	\STATE $\mathcal{D}$, buffer of transitions $(s, a, r, d, s')$, with $d$ the episode termination signal.
	\STATE $\deltat$ and $\gamma$, time discretization and discount factor.
	\STATE \textbf{nb\_epochs} number of epochs.
	\STATE \textbf{nb\_steps}, number of steps per epoch.
	\STATE
	\STATE Observe initial state $s^0$
	\STATE $t \gets 0$
	\FOR {$e=0, \textbf{nb\_epochs}$}
	\FOR {$j=1, \textbf{nb\_steps}$}
	\STATE $a^k \leftarrow \pi^{\text{explore}}(s^k, \nu^k_\deltat)$.
	\STATE Perform $a^k$ and observe $(r^{k+1}, d^{k+1}, s^{k+1})$.
	\STATE Store $(s^k, a^k, r^{k+1}, d^{k+1}, s^{k+1})$ in $\mathcal{D}$.
	\STATE $k \gets k + 1$
	\ENDFOR
	\FOR {$k=0, \text{nb\_learn}$}
	\STATE \text{Sample a batch of $N$ random transitions from $\mathcal{D}$}
	\STATE $Q^i \gets V_{\theta}(s^i) + \deltat\hspace{-.17em}\left(
	\bar{A}_{\psi}(s^i, a^i) - \bar{A}_{\psi}(s^i, \pi_\phi(s^i))\right)$
	\STATE $\tilde{Q^i} \gets r^i\deltat + (1 - d^i) \gamma^{\deltat} V_{\theta}(s'^i)$
	\STATE $\Delta \theta \gets \frac{1}{N}\sum\limits_{i=1}^N  \frac{\left(Q^i - \tilde{Q^i}\right)\partial_{\theta} V_{\theta}(s^i)}{\deltat}$
	\STATE $\Delta \psi \gets \frac{1}{N}\sum\limits_{i=1}^N \frac{\left(Q^i - \tilde{Q^i}\right)\partial_{\psi} \left(\bar{A}_{\psi}(s^i, a^i) - \bar{A}_{\psi}(s^i, \pi_\phi(s^i))\right) }{\deltat}$
	\STATE $\Delta \phi \gets \frac{1}{N} \sum\limits_{i=1}^N \partial_a \bar{A}_\psi(s^i, \pi_\phi(s^i)) \partial_\phi \pi_\phi(s^i)$
	\STATE Update $\theta$ with \textbf{opt}$_V$, $\Delta \theta$ and learning rate $\alpha^V \deltat$.
	\STATE Update $\psi$ with \textbf{opt}$_A$, $\Delta \psi$ and learning rate $\alpha^A \deltat$.
	\STATE Update $\phi$ with \textbf{opt}$_\pi$, $\Delta \phi$ and learning rate $\alpha^\pi \deltat$.
	\ENDFOR
	\ENDFOR
\end{algorithmic}

	\caption{Continuous DAU}
	\label{alg:cdau}
\end{algorithm}
\begin{algorithm}

\begin{algorithmic}
	\STATE \textbf{Inputs:}
	\STATE $\psi$ and $\phi$, parameters of
	$Q_\psi$ and $\pi_\phi$.
	\STATE $\psi'$ and $\phi'$, parameters of target networks
	$Q_{\psi'}$ and $\pi_{\phi'}$.
	\STATE $\pi^{\text{explore}}$ and $\nu$ defining an exploration policy.
	\STATE \textbf{opt}$_Q$, \textbf{opt}$_\pi$, $\alpha^Q$ and $\alpha^\pi$, optimizers and learning rates.
	\STATE $\mathcal{D}$, buffer of transitions $(s, a, r, d, s')$, with $d$ the episode termination signal.
	\STATE $\gamma$ discount factor.
	\STATE $\tau$ target network update factor.
	\STATE \textbf{nb\_epochs} number of epochs.
	\STATE \textbf{nb\_steps}, number of steps per epoch.
	\STATE
	\STATE Observe initial state $s^0$
	\STATE $t \gets 0$
	\FOR {$e=0, \textbf{nb\_epochs}$}
	\FOR {$j=1, \textbf{nb\_steps}$}
	\STATE $a^k \leftarrow \pi^{\text{explore}}(s^k, \nu^k)$.
	\STATE Perform $a^k$ and observe $(r^{k+1}, d^{k+1}, s^{k+1})$.
	\STATE Store $(s^k, a^k, r^{k+1}, d^{k+1}, s^{k+1})$ in $\mathcal{D}$.
	\STATE $k \gets k + 1$
	\ENDFOR
	\FOR {$k=0, \text{nb\_learn}$}
	\STATE \text{Sample a batch of $N$ random transitions from $\mathcal{D}$}
	\STATE $\tilde{Q^i} \gets r^i + (1 - d^i) \gamma Q_{\psi'}(s'^i, \pi_{\phi'}(s'^i))$
	\STATE $\Delta \psi \gets \frac{1}{N}\sum\limits_{i=1}^N \left(Q^i - \tilde{Q^i}\right) \partial_\psi Q(s^i, a^i)$
	\STATE $\Delta \phi \gets \frac{1}{N} \sum\limits_{i=1}^N \partial_a Q_\psi(s^i, \pi_\phi(s^i)) \partial_\phi \pi_\phi(s^i)$
	\STATE Update $\psi$ with \textbf{opt}$_Q$, $\Delta \psi$ and learning rate $\alpha^Q$.
	\STATE Update $\phi$ with \textbf{opt}$_\pi$, $\Delta \phi$ and learning rate $\alpha^\pi$.
	\STATE $\psi' \gets \tau \psi' + (1 - \tau) \psi$
	\STATE $\phi' \gets \tau \phi' + (1 - \tau) \phi$
	\ENDFOR
	\ENDFOR
\end{algorithmic}

	\caption{DDPG}
	\label{alg:ddpg}
\end{algorithm}
\begin{algorithm}

\begin{algorithmic}
	\STATE \textbf{Inputs:}
	\STATE $\psi$ parameter of
	$Q_\psi$.
	\STATE $\psi'$, parameters of target networks
	$Q_{\psi'}$.
	\STATE $\pi^{\text{explore}}$ and $\nu$ defining an exploration policy.
	\STATE \textbf{opt}$_Q$, $\alpha^Q$ optimizer and learning rate.
	\STATE $\mathcal{D}$, buffer of transitions $(s, a, r, d, s')$, with $d$ the episode termination signal.
	\STATE $\gamma$ discount factor.
	\STATE $\tau$ target network update factor.
	\STATE \textbf{nb\_epochs} number of epochs.
	\STATE \textbf{nb\_steps}, number of steps per epoch.
	\STATE
	\STATE Observe initial state $s^0$
	\STATE $t \gets 0$
	\FOR {$e=0, \textbf{nb\_epochs}$}
	\FOR {$j=1, \textbf{nb\_steps}$}
	\STATE $a^k \leftarrow \pi^{\text{explore}}(s^k, \nu^k)$.
	\STATE Perform $a^k$ and observe $(r^{k+1}, d^{k+1}, s^{k+1})$.
	\STATE Store $(s^k, a^k, r^{k+1}, d^{k+1}, s^{k+1})$ in $\mathcal{D}$.
	\STATE $k \gets k + 1$
	\ENDFOR
	\FOR {$k=0, \text{nb\_learn}$}
	\STATE \text{Sample a batch of $N$ random transitions from $\mathcal{D}$}
	\STATE $\tilde{Q^i} \gets r^i + (1 - d^i) \gamma \max\limits_{a'}Q_{\psi'}(s'^i, a')$
	\STATE $\Delta \psi \gets \frac{1}{N}\sum\limits_{i=1}^N \left(Q^i - \tilde{Q^i}\right) \partial_\psi Q(s^i, a^i)$
	\STATE Update $\psi$ with \textbf{opt}$_Q$, $\Delta \psi$ and learning rate $\alpha^Q$.
	\STATE $\psi' \gets \tau \psi' + (1 - \tau) \psi$
	\ENDFOR
	\ENDFOR
\end{algorithmic}

	\caption{DQN}
	\label{alg:dqn}
\end{algorithm}

For DDPG and DQN, two different settings were experimented with:
\begin{itemize}
	\item One with time discretization scalings, to keep the comparison
		fair. In this setting, the discount factor is still scaled as $\gamma^\deltat$,
		rewards are scaled as $r \deltat$, and learning rates are scaled to obtain parameter
		updates of order $\deltat$. As RMSprop is used for all experiments, this amounts
		to using a learning rate scaling as $\alpha^Q = \tilde{\alpha}^Q \deltat$,
		$\alpha^\pi = \tilde{\alpha}^\pi \deltat$.
	\item One without discretization scalings. In that case, only the discount factor is scaled
		as $\gamma^\deltat$, to prevent unfair shortsightedness. All other
		parameters are set with a reference $\deltat_0 = 1e-2$. For instance,
		for all $\deltat$'s, the reward perceived is $r * \deltat_0$, and
		similarily for learning rates, $\alpha^Q = \tilde{\alpha}^Q
		\deltat_0$, $\alpha^\pi = \tilde{\alpha}^Q \deltat_0$. These scalings
		don't depend on the discretization, but perform decently at least for
		the highest discretization.
\end{itemize}
\subsection{Global hyperparameters}
The following hyperparameters are maintained constant throughout all our experiments,
\begin{itemize}
	\item All networks used are of the form
		\begin{verbatim}
		Sequential(
		    Linear(nb_inputs, 256),
		    LayerNorm(256),
		    ReLU(),
		    Linear(256, 256),
		    LayerNorm(256),
		    ReLU(),
		    Linear(256, nb_outputs)
		).
		\end{verbatim}
	Policy networks have an additional $\tanh$ layer to constraint action range. On certain
	environments, network inputs are normalized by applying a mean-std normalization, with
	mean and standard deviations computed on each individual input features, on all previously
	encountered samples.
	\item $\mathcal{D}$ is a cyclic buffer of size $1000000$.
	\item $\textbf{nb\_steps}$ is set to $10$, and $256$ environments are run in parallel to
		accelerate the training procedure, totalling $2560$ environment interactions between
		learning steps.
	\item $\textbf{nb\_learn}$ is set to $50$.
	\item The physical $\gamma$ is set to $0.8$. It is always scaled as $\gamma^\deltat$ (even for
		unscaled DQN and DDPG).
	\item $N$, the batch size is set to $256$.
	\item RMSprop is used as an optimizer without momentum, and with
		$\alpha=1 - \deltat$ (or $1 - \deltat_0$ for unscaled DDPG and
		DQN).
	\item Exploration is always performed as described in the main text. The OU process used as
		parameters $\kappa = 7.5$, $\sigma = 1.5$.
	\item Unless otherwise stated, $\alpha_1 \deq \tilde{\alpha}^Q = \alpha^V = \alpha^A = 0.1$, $\alpha_2 \deq \tilde{\alpha}^\pi =
		\alpha^\pi = 0.03$.
	\item $\tau = 0.9$
\end{itemize}
\subsection{Environment dependent hyperparameters}
We hereby list the hyperparameters used for each environment. Continuous actions environments are marked with a
(C), discrete actions environments with a (D).
\begin{itemize}
	\item {\bf Ant (C)}: State normalization is used. Discretization range: $[0.05, 0.02, 0.01, 0.005, 0.002]$.
	\item {\bf Cheetah (C)}: State normalization is used. Discretization range: $[0.05, 0.02, 0.01, 0.005, 0.002]$
	\item {\bf Bipedal Walker (C)\footnote{
				The reward for Bipedal Walker is modified not to scale with $\deltat$. This does not introduce any change for the default setup.
		}}: State normalization is used, $\alpha_2 = 0.02$. Discretization range: $[0.01, 0.005, 0.002, 0.001]$.
	\item {\bf Cartpole (D)}: $\alpha_2 = 0.02$, $\tau = 0$. Discretization range: $[0.01, 0.005, 0.002, 0.001, 0.0005]$.
	\item {\bf Pendulum (C)}: $\alpha_2 = 0.02$, $\tau = 0$. Discretization range: $[0.01, 0.005, 0.002, 0.001, 0.0005]$.

\end{itemize}

\section{Additional results}
Additional results mentionned in the text are presented in this section.
\newcommand{\cw}{\textwidth}
\begin{figure*}[h]
	\includegraphics[width=\cw]{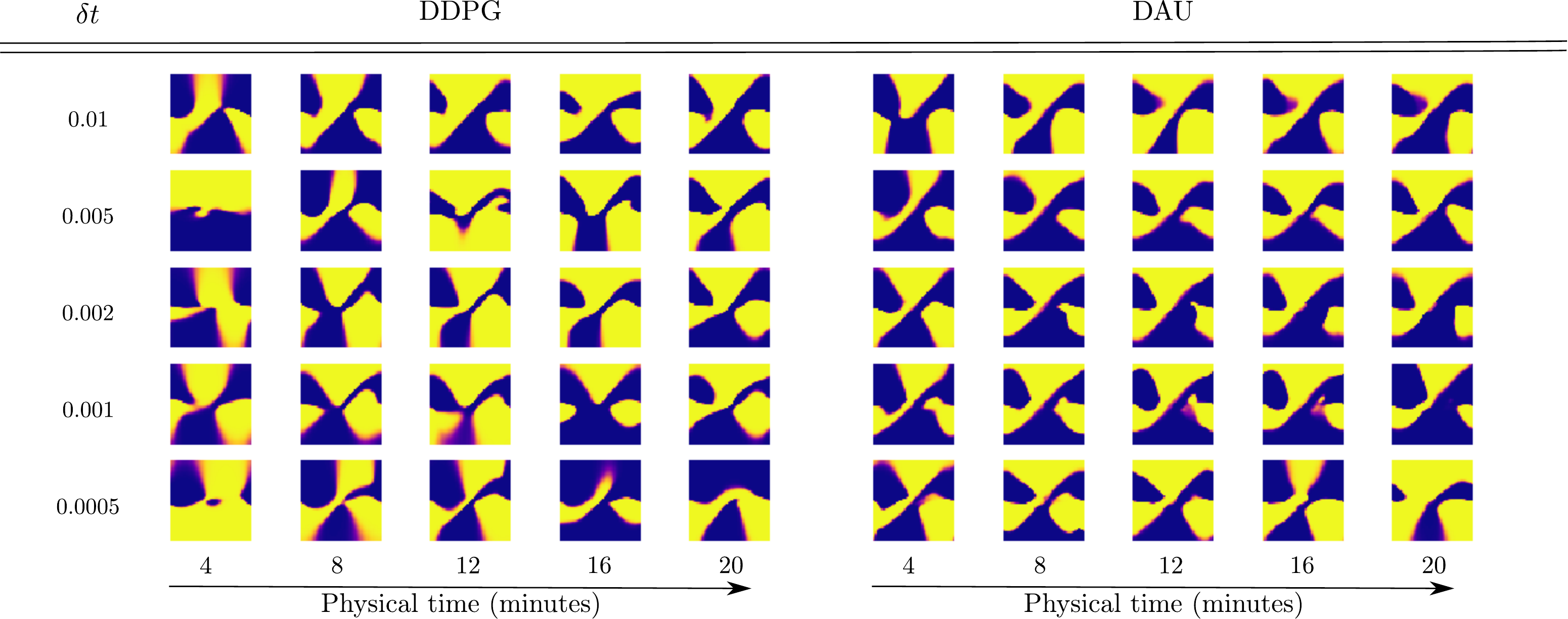}
	\caption{Policies obtained by DDPG (unscaled version) and AU at different instants in physical time of training on the pendulum swing-up environment. Each image represents the policy learnt by the policy network, with $x$-axis representing angle, and $y$-axis angular velocity. The lighter the pixel, the closer to $1$ the action, the darker, the closer to $-1$.}
	\label{fig:pend1}
\end{figure*}
\begin{figure*}[h]
	\includegraphics[width=\cw]{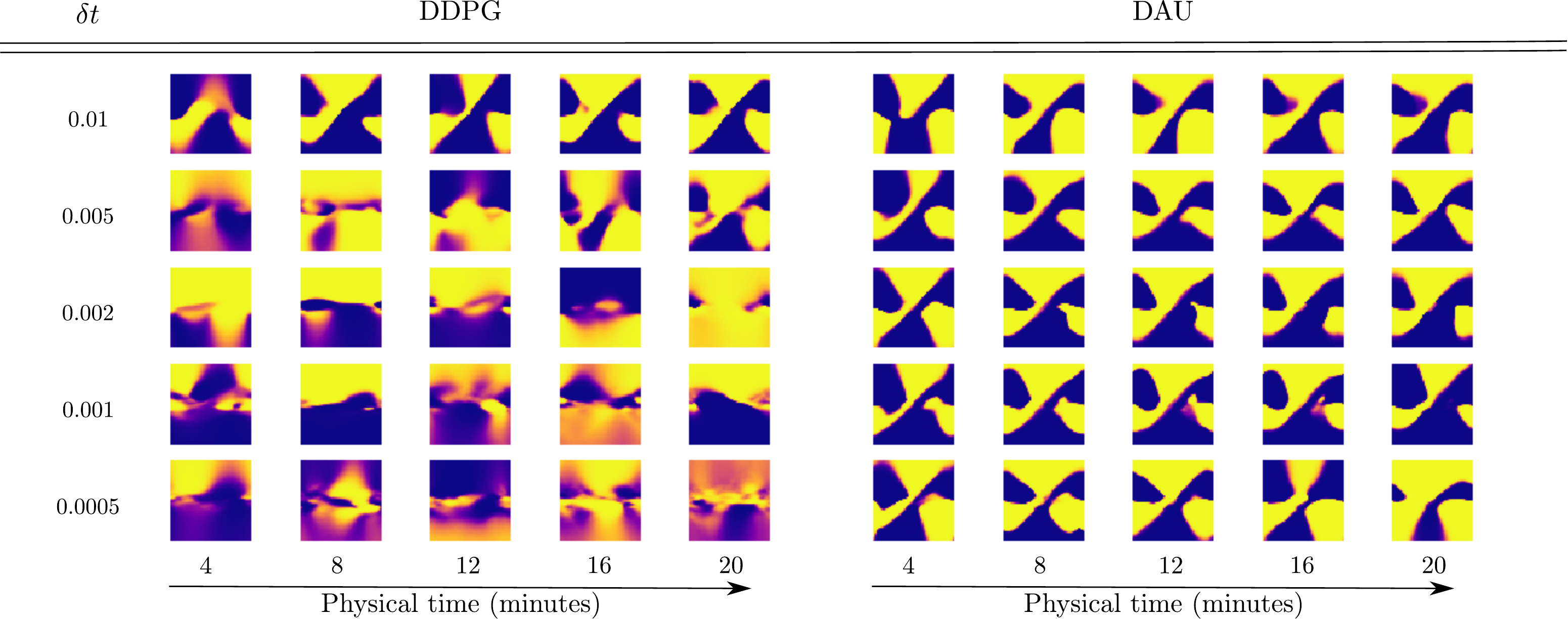}
	\caption{Policies obtained by DDPG (scaled version) and AU at different instants in physical time of training on the pendulum swing-up environment. Each image represents the policy learnt by the policy network, with $x$-axis representing angle, and $y$-axis angular velocity. The lighter the pixel, the closer to $1$ the action, the darker, the closer to $-1$.}
	\label{fig:pend2}
\end{figure*}
\begin{figure*}[h]
	\includegraphics[width=\cw]{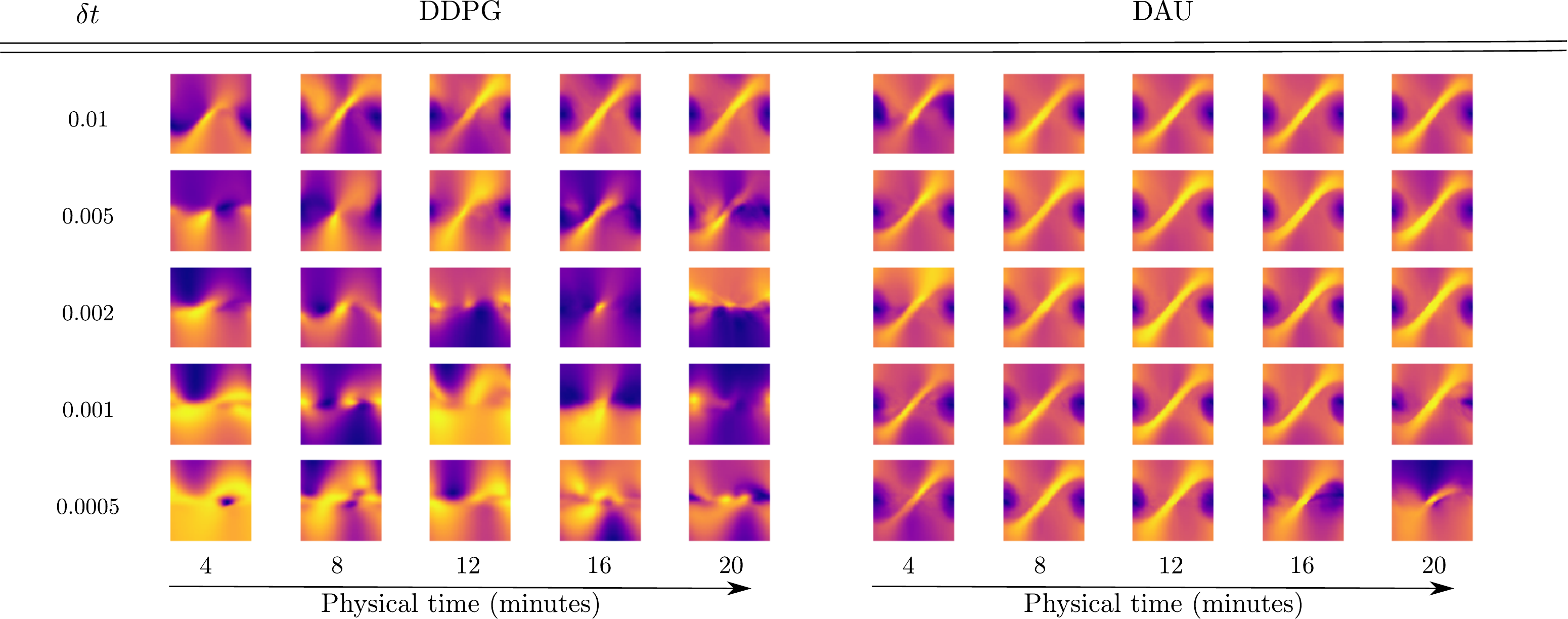}
	\caption{Value functions obtained by DDPG (scaled version) and AU at different instants in physical time of training on the pendulum swing-up environment. Each image represents the value function learnt, with $x$-axis representing angle, and $y$-axis angular velocity. The lighter the pixel, the higher the value.}
	\label{fig:pend3}
\end{figure*}
\begin{figure*}[h]
	\centering
	\includegraphics[width=\cw]{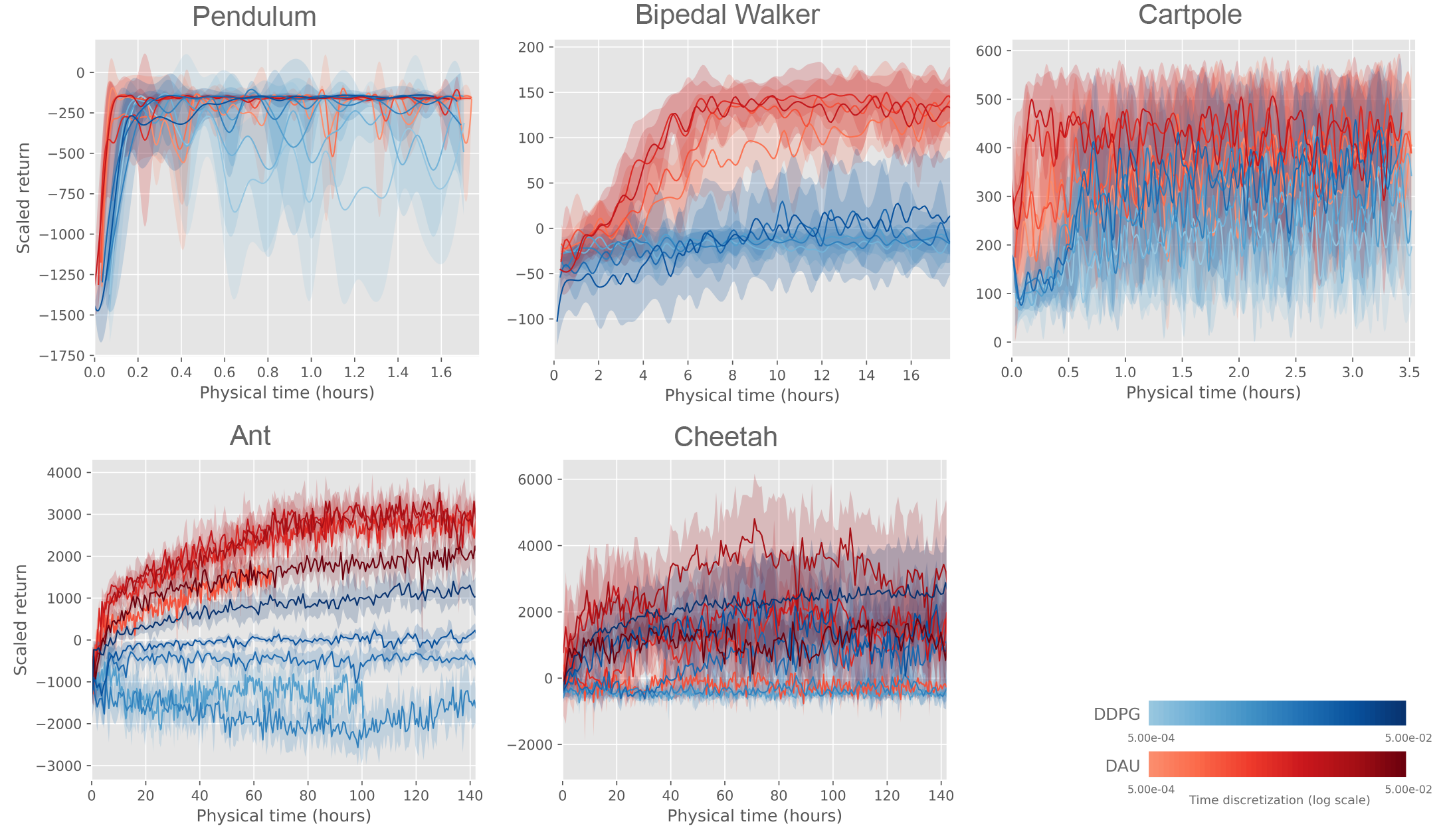}
	\label{fig:full_results}
	\caption{Learning curves for DAU and DDPG (scaled) on classic control benchmarks for  various time discretization $\deltat$: Scaled return as a function of the physical time spent in the environment.}
\end{figure*}


\end{document}